\documentclass{article}
\usepackage[margin=.75in]{geometry}

\usepackage{graphicx}
\usepackage{amsmath}
\usepackage{amssymb}
\usepackage{amsthm}
\usepackage{booktabs}
\usepackage{enumitem}
\usepackage{hyperref}
\usepackage{url}
\usepackage{microtype}
\usepackage{wrapfig}

\PassOptionsToPackage{numbers,sort&compress}{natbib}
\usepackage{natbib}

\usepackage[table]{xcolor}
\definecolor{ourscolor}{RGB}{230,243,255}  
\definecolor{mydarkblue}{RGB}{22,55,92}
\usepackage[most]{tcolorbox}

  \newtcolorbox{summarybox}{
    colback=ourscolor,
    colframe=mydarkblue,
    boxrule=1pt,
    arc=3pt,
    left=6pt,
    right=6pt,
    top=6pt,
    bottom=6pt,
  }

\newtheorem{theorem}{Theorem}
\newtheorem{proposition}{Proposition}
\newtheorem{corollary}{Corollary}
\newtheorem{lemma}{Lemma}
\newtheorem{definition}{Definition}
\newtheorem{remark}{Remark}

\newcommand{\polar}[1]{\mathrm{polar}(#1)}

\title{The Spectral Dynamics and Noise Geometry of Muon}

\author{%
  Pierfrancesco Beneventano \quad Mahmoud Abdelmoneum \quad Tomaso Poggio \\
  Massachusetts Institute of Technology\\
  \texttt{$\{$pierb,mabdel03$\}$@mit.edu}
}

\begin{document}

\maketitle

\begin{summarybox}
\centering
\textit{\color{mydarkblue}\textbf{Paper written by prompting pAI/MSc~\citep{abdelmoneum2026paimsc}.}}\\
\textit{\color{mydarkblue}\textbf{No parts were written, thought of, or performed by humans except Section \ref{sec:pai}.}}
\end{summarybox}

\begin{abstract}
Muon replaces a matrix gradient \(G=U\Sigma V^\top\) by its polar factor \(UV^\top\). This keeps the singular directions selected by the gradient, but makes the update spectrum flat. We study the optimization bias created by this operation. Under explicit alignment assumptions, we prove that the polar update is the one-step entropy-maximizing choice among bounded updates that use the gradient singular directions and do not adapt to the current weight spectrum. In an underdetermined regression model, we derive exact singular-value dynamics for continuous-time Muon and identify a measurement-dependent condition under which the normalized spectrum moves toward equal nonzero singular values. This geometry also rules out a common low-rank interpretation: at fixed Frobenius norm, Muon's distinguished state has a flat spectrum, whereas nuclear-norm minimization favors spectral concentration. Controlled matrix-sensing experiments separate the effect from simple gradient rescaling, show that norm-matched gradient descent does not reproduce Muon, and recover the predicted flattening trend across broad ablations. In small NanoGPT pretraining, Muon preserves stable rank, has a broad learning-rate plateau, and improves validation loss relative to AdamW; in a matched small-ViT control, the ranking reverses. The resulting picture is regime-dependent: Muon is not universally superior, but its flat-spectrum bias can help when many spectral directions need to remain active.
\end{abstract}
\section{Introduction}
\label{sec:intro}

A matrix gradient carries two kinds of information: which singular directions are active, and how strongly each direction appears. Gradient descent uses both. Muon~\citep{jordan2024muon} keeps the directions and discards the relative strengths. If
\[
    G=U\Sigma V^\top,
\]
then Muon updates with the polar factor
\[
    \polar{G}=UV^\top.
\]
Thus Muon is not merely rescaling the gradient. It changes the update spectrum: a gradient-descent step inherits the singular values of \(G\), while a Muon step gives every active singular direction the same amplitude.

This paper studies the bias induced by that flattening operation. Muon has become a practical optimizer for transformer training~\citep{liu2025moonlight}, but empirical acceleration alone does not explain what the polar map selects. Two natural explanations are tempting. One is that Muon is just matrix-gradient normalization; another is that spectral-norm geometry should lead to nuclear-norm, hence low-rank, implicit bias. We show that both are incomplete. Muon is better understood as a \emph{flat-update-spectrum} method: it keeps the gradient singular directions while allocating equal update amplitude across them.

The regression setting makes this question sharp. In separable classification, recent theory for Muon-like methods studies max-margin directions at infinity. In underdetermined regression, the zero-loss set is instead an affine interpolation manifold, and the optimizer must choose one finite interpolant among many. We ask which solution spectra are favored by the polar update.

Our main theoretical model is the projected polar flow on
\[
    \mathcal M=\{W:XW=Y\}.
\]
Under a shared-frame alignment assumption, the singular values obey
\[
    \dot\sigma_i=-\eta\,\alpha_i,
    \qquad
    \alpha_i\in[0,1],
\]
with \(\alpha_i\) determined by the measurement geometry. This yields an exact sign criterion for whether the normalized spectrum flattens or concentrates. The distinction matters: the pairwise potential \(R_{\mathrm{pw}}\) identifies the flat spectrum as the unique fixed-radius variational state, but it is not a Lyapunov function for the unnormalized flow. Muon contracts the norm; whether its spectral shape flattens is governed by the sign criterion.

\begin{wrapfigure}{r}{0.54\textwidth}
\vspace{-10pt}
\centering
\includegraphics[width=\linewidth]{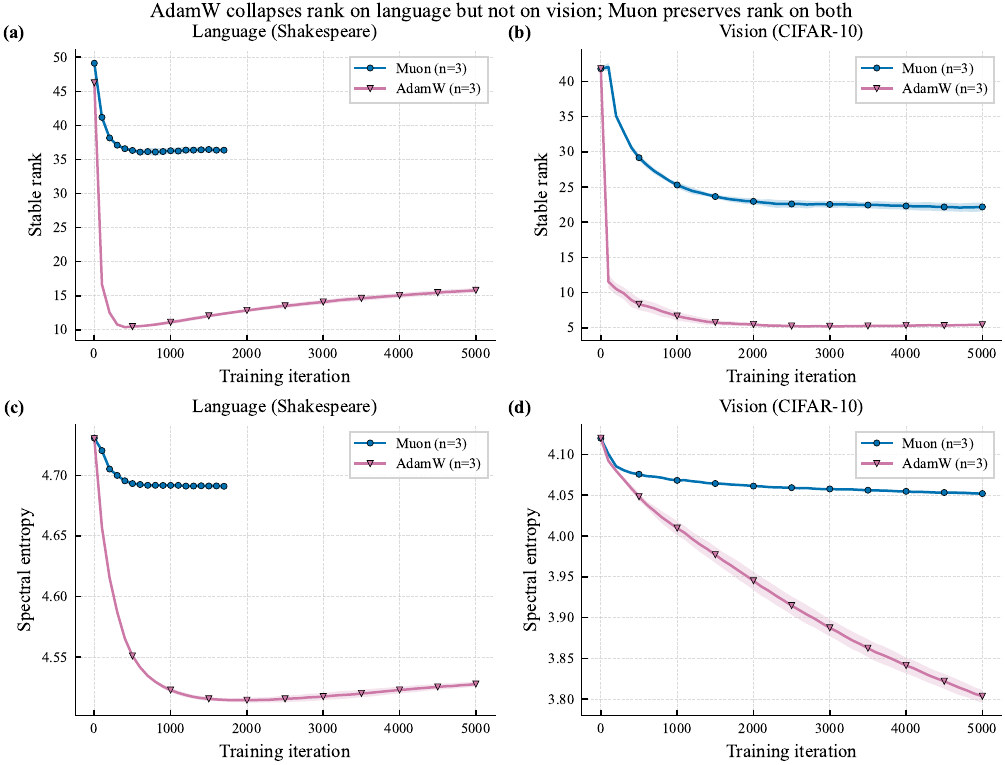}
\vspace{-8pt}
\caption{\textbf{Regime reversal.}
Muon preserves stable rank in small NanoGPT, while AdamW wins in a matched small-ViT/CIFAR-10 control. 
This is directional evidence for regime dependence, not a pure modality intervention.}
\label{fig:hero}
\vspace{-0.6cm}
\end{wrapfigure}

\paragraph{Contributions.}
We make four contributions.
\begin{itemize}[leftmargin=*,topsep=2pt,itemsep=2pt,parsep=0pt]
    \item \textbf{A one-step spectral-bias theorem.}
    Under explicit alignment assumptions, we prove that the polar profile maximizes first-order spectral-entropy gain among bounded updates that follow the gradient singular directions and ignore the current weight spectrum.

    \item \textbf{Trajectory-level singular-value dynamics.}
    We derive exact singular-value dynamics for the projected polar flow, including a measurement-geometry sign criterion for spectral flattening and a fixed-radius characterization of the flat spectrum.

    \item \textbf{Separation from normalization and nuclear-norm explanations.}
    In controlled matrix-sensing experiments, norm-matched gradient descent does not reproduce Muon's behavior, and Muon's limit points do not match convex nuclear-norm minimizers.

    \item \textbf{Transformer signatures and regime dependence.}
    In small NanoGPT pretraining, Muon preserves stable rank; in an architecture- and budget-matched small-ViT control, the optimizer ranking reverses. These experiments suggest that the flat-spectrum bias is useful in some spectral regimes but not universally.
\end{itemize}

\begin{summarybox}
{\color{mydarkblue}\textbf{Takeaway. }}The resulting practitioner-relevant message is regime-dependent, not universal. Flat update spectra can help when many spectral directions must remain active; they can be unnecessary or worse when the useful spectrum is already low-dimensional or well conditioned. 
\end{summarybox}

\textit{In small NanoGPT pretraining, Muon preserves stable rank and improves validation loss relative to AdamW. In a matched small-ViT control, the ranking reverses.}
This supports a regime story rather than a universal optimizer claim.

\newpage
\begin{summarybox}
\section{How AI Wrote this Paper}
\label{sec:pai}
This section is entirely written by humans and this is the only such section of this paper.

\subsection{We seek feedback}
\label{sec:we_seek_feedback}

AI-assisted research can make a project appear more settled than it is. 
Messeri and Crockett, \citep{messeri2024artificial}, describe a broad epistemic risk of AI in science: researchers may develop \emph{illusions of understanding}, believing they understand more, have explored more, or have achieved more objectivity than is warranted. 
At the artifact level, Beneventano et al., \citep{beneventano2026agent}, call the analogous failure mode \emph{closure failure}: a manuscript may linguistically resolve uncertainty before the underlying claims have been adequately verified.

We tried to mitigate this risk by treating AI-assisted drafts as provisional. 
The authors manually inspected theorem statements, proof dependencies, numerical summaries, experimental logs, and novelty claims. 
The agentic system was asked to weaken or remove the claims supported only by conditional theory, small-scale experiments, or incomplete evidence.

We are actively seeking feedback from researchers working on Muon, matrix optimization, implicit bias, optimizer theory, and language-model pretraining. 
We especially welcome scrutiny of the proof assumptions, experimental interpretation, novelty positioning, and whether any abstract-level claims outrun their available verifier.

\subsection{How we produced it}
We are aware that human researchers already use AI heavily to produce papers. However, they typically exchange hundreds or thousands of prompts and end up mostly writing the text themselves. This is what we tried to avoid. This paper was the testbed for us to develop pAI/MSc \citep{abdelmoneum2026paimsc}, we iterated on the agentic system towards higher "quality" and "taste" benchmarking it on how it was writing this paper.

The initial prompt for this paper was the following:
\begin{quote}
    We propose a systematic investigation into the theoretical and empirical properties of the Muon optimizer, focusing on two foundational questions. First, we ask whether Muon induces an implicit regularization effect analogous to known results for SGD and Adam — specifically, whether training under Muon on a given loss landscape is equivalent to gradient descent on a regularized objective (i.e., whether there exist loss-algorithm pairs ($L_1$, Muon) $\approx$ ($L_2$, GD) for some implicitly regularized loss $L_2$). Establishing such a correspondence would provide principled insight into Muon’s inductive biases and generalization behavior, and constitutes a standalone theoretical contribution. Second, we investigate whether Muon admits a scaling rule relating batch size to step size — that is, whether there exists a critical batch size regime and a precise functional relationship governing how step size must adjust with batch size to preserve per-step loss reduction. Such scaling laws are well-characterized for SGD but remain unexplored for Muon, and understanding them is essential for efficient large-scale training. \\

    To ground these theoretical inquiries empirically, we conduct a controlled comparison of Muon against Adam(W) on language modeling using the NanoGPT framework. A key methodological priority is ensuring fairness in this comparison; we will perform a thorough literature review of best practices for optimizer benchmarking — including matched compute budgets, hyperparameter tuning protocols (e.g., grid search vs. population-based training), and normalization of effective learning rates — to avoid conflating optimizer-intrinsic effects with tuning artifacts. Together, these three threads — implicit regularization theory, scaling law analysis, and rigorous empirical validation — aim to provide a comprehensive characterization of Muon’s optimization dynamics, bridging the gap between its emerging practical adoption and formal theoretical understanding.
\end{quote}

Precisely, the pipeline was to give this prompt, then review the outcome as a human, give the review to the agentic system again and checking its iteration 2, and so on.

The paper you are about to read is based on what came up at the end of the 5th iteration. Precisely, 5 main iterations of pAI/MSc, plus one of our pAI/MSc claude skill to polish, clean it, make it "submittable" to Arxiv, and making it more readable to humans. The details of this pipeline can be found in \citep{abdelmoneum2026paimsc}.

The process started in \textbf{\textit{February 2026}}, the day we obtain this version was \textbf{\textit{April 9th, 2026}}. 

\end{summarybox}

\newpage

\section{Related Work}
\label{sec:related}

\paragraph{Implicit bias in matrix sensing.}
The implicit bias of gradient descent in overparameterized regression is well understood: GD converges to the minimum Euclidean-norm solution in linear models \citep{soudry2017bias,gunasekar2017implicit} and to the minimum nuclear-norm solution in matrix sensing under near-origin initialization \citep{gunasekar2017implicit,wu2021implicit}. The standard mirror-descent guarantee for preconditioned dynamics \citep{gunasekar2018implicitly} requires the regularizer to admit a Legendre potential; the polar map does not arise from a Bregman projection, so this route is unavailable, and Theorem~\ref{thm:thmA} proceeds via a direct Riemannian analysis instead.

\paragraph{Muon and spectral optimization.}
Muon \citep{jordan2024muon} computes the polar factor of the gradient via Newton--Schulz iteration \citep{bernstein2024spectral,tuddenham2022orthogonalising}. 
On the optimization side, Muon has been interpreted as a non-Euclidean steepest-descent or trust-region method under spectral-norm geometry \citep{lihong2025muon,kovalev2025trustregion}, as an instance of a constrained spectral-norm implicit-regularization framework \citep{chen2025spectralconstraint}. \citep{higham2008functions,amsel2025polar,lau2025polargrad,davis2025spectral} characterize the polar/PolarGrad family. \citep{shen2025convergence} prove convergence rates under standard nonconvex assumptions. \citep{liu2025moonlight} demonstrate Muon's empirical advantages in LLM pretraining. Empirical accounts of spectral dynamics during training \citep{yunis2024spectral,olsen2025spectra} are consistent with the flat-spectrum prediction of Theorem~\ref{thm:thmA}(iii)--(iv) under the sign-criterion regime.

\paragraph{Concurrent classification work.}
\citep{fan2025muon,gronich2026muon} prove max-margin convergence of Muon under cross-entropy; \citep{kang2026muon} report uniform spectral growth in LoRA-style factorization; \citep{vasudeva2025imbalanced} prove equal-rate principal-component learning in bilinear classification. These results are structurally complementary: classification geometry is loss-divergent and margin-controlled, while ours is loss-vanishing and manifold-controlled. A scope table comparing this paper to the concurrent work appears in Appendix~\ref{app:scope_comparison} (Table~\ref{tab:scope}).

\paragraph{Convergence theory and practical Muon variants.}
Several works establish convergence guarantees for Muon or Muon-like methods under smooth nonconvex assumptions, with variants incorporating Nesterov momentum, weight decay, adaptive scaling, trust-ratio normalization, or schedule-free averaging \citep{shen2025convergence,sato2025cbs,si2025adamuon,lou2026orscale,kim2026amuse,cheng2026trasmuon}. These results reinforce a useful separation of regimes: Muon's empirical success appears to depend on smooth, matrix-structured geometry rather than on generic nonsmooth convex theory. Our theorems are deliberately local and structural: they assume a fixed-rank stratum, projected interpolation dynamics, and explicit spectral-gap or polar-misalignment controls.

\paragraph{When spectral updates help in deep networks (by ChatGPT 06/06/26).}
Empirical and theoretical studies of neural-network spectra suggest that training often has a persistent low-rank-plus-bulk structure in the singular values of weights and activations \citep{yunis2024spectral,olsen2025spectra,lauditi2026spectral}. \citep{davis2025spectral} give a layerwise condition predicting when a spectral update should produce a larger one-step decrease than a Euclidean update, comparing the gradient's nuclear-to-Frobenius ratio with the stable rank of incoming activations. Large-scale empirical reports further show that Muon can improve pretraining efficiency and large-batch behavior in language models \citep{essentialai2025practical,liu2025moonlight}. Our transformer experiments should be read in this spirit: they are spectral signatures consistent with the theory, not a proof that transformer training literally reduces to the affine matrix-regression model studied in Theorem~\ref{thm:thmA}.


\paragraph{Critical batch size and polar-map noise geometry.}
The critical-batch-size literature studies the signal-to-noise crossover at which increasing the batch size yields diminishing returns \citep{mccandlish2018empirical,smith2018dont}. Recent LLM work revisits this crossover empirically and through scaling laws \citep{zhang2024cbs,merrill2025cbs}. For Muon specifically, \citep{sato2025cbs} derive a critical batch size from a nonconvex convergence-rate analysis. Our Theorem~\ref{thm:thmB} gives a different, one-step linearized derivation: in the square full-rank regime, the stochastic sensitivity of the polar map is controlled by
\[
        \operatorname{tr}\!\left(DP[\Sigma]DP^\top\right),
        \qquad
        S(\mu)=\sum_{i\neq j}(\sigma_i+\sigma_j)^{-2}.
\]
Thus the batch-size crossover is tied directly to polar-map curvature. 

\paragraph{Summary of distinction.}
The closest overlapping papers explain why flat update spectra can improve stability, convergence, or component-wise learning rates. This paper addresses a different selection problem. In underdetermined regression, after the loss has vanished, the polar flow must choose one finite interpolant from an affine manifold. Our contribution is to characterize that choice through projected polar dynamics, a measurement-dependent flattening criterion, a flat-spectrum variational state, and a polar-map noise-sensitivity formula. This separates Muon's implicit bias from both norm-matched gradient normalization and nuclear-norm minimization.


\paragraph{April--May 2026 Muon papers (by ChatGPT 06/06/26).}
Several very recent preprints appeared between when this research effort was concluded and now,
sharpening the surrounding picture of Muon and polarized matrix updates. Newton--Muon derives
a layerwise quadratic-surrogate view and interprets standard Muon as neglecting an input-covariance
right preconditioner \citep{du2026newtonmuon}; Muon$^2$ preconditions the momentum matrix by
Adam-style second moments before orthogonalization to improve Newton--Schulz conditioning
\citep{liu2026muon2}. Closest in language to our work, recent spectral-flattening theory argues
that Muon's larger stable learning rates and improved convergence can be explained under a
Kronecker-factored curvature model \citep{nguyen2026spectralflattening}; this is a training-stability
and descent-rate theory, whereas Theorem~\ref{thm:thmA} studies finite-interpolant selection on the
zero-loss affine manifold. Complementarily, random-spectrum and inverted-spectrum variants have
been reported to perform comparably to Muon, emphasizing local alignment and descent potential
rather than a unique global norm geometry \citep{shumaylov2026muonnotthat}. Negative and
regime-dependent results further delimit the scope of Muon theory: Muon need not converge on
general convex Lipschitz objectives without error feedback \citep{parshakova2026muonconvex};
uniform spectral whitening can fail in VLA and RLVR settings, motivating high-pass Newton--Schulz
filters \citep{fan2026pion}; and optimizer--recipe and robustness-dependent spectral effects have
been reported in ViTs and adversarial training \citep{southworth2026muonvit,yan2026adversarialmuon}.
Finally, DP-Muon analyzes clipping, Gaussian privacy noise, momentum, and Newton--Schulz
orthogonalization \citep{kim2026dpmuon}; this is adjacent to our noise analysis, but studies
privacy-induced heat-smoothing bias rather than the Frechet polar sensitivity $S(\mu)$ and the
critical-batch-size crossover of Theorem~\ref{thm:thmB}.
\section{Setup}
\label{sec:setup}

\paragraph{Matrix regression and the polar map.}
Let $W\in\mathbb{R}^{m\times n}$ be a weight matrix and $X\in\mathbb{R}^{p\times m}$ a measurement matrix with $\operatorname{rank}(X)=p<m$ (so the system is underdetermined). The matrix-regression loss is $L(W)=\tfrac{1}{2}\|XW-Y\|_F^2$. Write $G=\nabla_W L(W)=X^\top(XW-Y)$ with SVD $G=U_G\Sigma_G V_G^\top$; the polar map is $P(G):=U_G V_G^\top$. The continuous-time Muon flow with step size $\eta>0$ is
\begin{equation}
\dot W \;=\; -\eta\,P(G).
\label{eq:muon_continuous}
\end{equation}

\paragraph{The interpolation manifold.}
Let $\mathcal{M}=\{W\in\mathbb{R}^{m\times n}:XW=Y\}$ denote the affine manifold of zero-loss interpolants. Its tangent space is constant: $T_W\mathcal{M}=\ker X$ for all $W\in\mathcal{M}$. The orthogonal projector onto $T_W\mathcal{M}$ in the Frobenius inner product is
\[
P_\perp \;:=\; I_m - X^\dagger X, \qquad X^\dagger := X^\top(XX^\top)^{-1}.
\]
$P_\perp$ is symmetric, idempotent, with $\|P_\perp\|_{\mathrm{op}}=1$.

\paragraph{Shared-frame condition.}
We invoke the following structural condition.

\begin{definition}[Shared Frames, abbreviated \textsf{SF}]
\label{def:sf}
The gradient and weight share singular frames: $U_G=U_W$ and $V_G=V_W$, so $G=U_W\Sigma_G V_W^\top$. Under \textsf{SF}, $P(G)=U_W V_W^\top = P(W)$, and the polar update acts diagonally on the singular values.
\end{definition}

\textsf{SF} is a structural condition; it is not assumed to hold globally in time. Section~\ref{sec:experiments} reports $\|\sin\Theta\|_F<0.1$ (Frobenius distance between gradient and weight singular frames) after an initial transient across all tested configurations. The Robustness Theorem (Theorem~\ref{thm:robust}) below quantifies how the conclusions of Theorem~\ref{thm:thmA} degrade when \textsf{SF} is only approximately satisfied.

\paragraph{Projected polar flow vs.\ literal Muon at zero loss.}
On the interpolation manifold $\mathcal M$, the gradient vanishes ($G=0$), so the literal continuous flow~\eqref{eq:muon_continuous} stops; in practice, however, the discrete Muon iterate carries a momentum buffer $m_t$ with $P(m_t)\ne 0$ and remains spectrally active. The object analyzed in Theorem~\ref{thm:thmA} is therefore not literal unconstrained Muon at zero loss, but the projected polar flow obtained by Euler-projecting the polar step back onto $\mathcal M$.

\begin{proposition}[Projected polar flow]
\label{prop:proj_flow}
Fix $W\in\mathcal M$. The orthogonal projection of one Euler step onto $\mathcal M$ is
\begin{equation}
W^+ \;=\; \Pi_{\mathcal M}\!\bigl(W-\eta\,P(W)\bigr) \;=\; W - \eta\,P_\perp\,P(W),
\label{eq:proj_step}
\end{equation}
where $\Pi_{\mathcal M}$ denotes Euclidean projection onto $\mathcal M$. The continuous-time flow is
\begin{equation}
\dot W \;=\; -\eta\,P_\perp\,P(W).
\label{eq:proj_flow}
\end{equation}
\end{proposition}

\begin{proof}
$\mathcal M$ is affine with constant tangent space $T_W\mathcal M=\ker X$, so $\Pi_{\mathcal M}(W+V)=W+P_\perp V$ for any $W\in\mathcal M$ and $V\in\mathbb R^{m\times n}$. Setting $V=-\eta P(W)$ gives \eqref{eq:proj_step}. Idempotence of $P_\perp$ ensures $W^+\in\mathcal M$. Equation~\eqref{eq:proj_flow} is the corresponding continuous-time vector field obtained by taking the per-step ODE limit $W^+ \approx W + \dot W\,\Delta t$ with $\Delta t = 1$.
\end{proof}

Theorem~\ref{thm:thmA} below characterizes the projected polar flow~\eqref{eq:proj_flow}. We use ``projected polar flow'' and ``projected self-polar flow'' interchangeably (both refer to~\eqref{eq:proj_flow}); ``gradient-driven flow'' refers to $\dot W=-\eta P_\perp P(G)$, the object of Theorem~\ref{thm:robust}.

\paragraph{The pairwise spectral functional.}
Let $W=U_W\Sigma_W V_W^\top$ be the SVD of $W$ with ordered singular values $\sigma_1\geq\cdots\geq\sigma_r>0$ and $r=\operatorname{rank}(W)$. We use two functionals:
\begin{equation}
R_{\mathrm{pw}}(W) \;:=\; -\sum_{i\neq j}\log(\sigma_i+\sigma_j),
\qquad
\widetilde R_{\mathrm{pw}}(W) \;:=\; -\sum_{i\neq j}\log(\widetilde\sigma_i+\widetilde\sigma_j),
\label{eq:Rpw}
\end{equation}
where $\widetilde\sigma_i:=\sigma_i/\rho$ and $\rho:=\|W\|_F$. The functional $\widetilde R_{\mathrm{pw}}$ is scale-invariant and tracks spectral \emph{shape}; flat spectra ($\widetilde\sigma_i=1/\sqrt r$) minimize $\widetilde R_{\mathrm{pw}}$ subject to $\sum_i\widetilde\sigma_i^2=1$.

\paragraph{Polar-map spectral sensitivity.}
\begin{definition}[Polar-Map Spectral Sensitivity]
\label{def:smu}
$S(\mu):=\sum_{i\neq j}(\sigma_i+\sigma_j)^{-2}$, evaluated at the gradient singular values $\{\sigma_i\}$. $S(\mu)$ measures how strongly the polar map amplifies noise: small or near-degenerate singular values inflate $S(\mu)$ via the $(\sigma_i+\sigma_j)^{-2}$ terms.
\end{definition}

For $r$ equal singular values $\sigma_0$, $S(\mu)=r(r-1)/(4\sigma_0^2)$; Appendix~\ref{sec:methods:smu_validation} verifies this against $\operatorname{tr}(D_P D_P^\top)/2$.

\begin{table}[h]
\centering\small
\caption{The three flows referenced throughout the paper. Theorem~\ref{thm:robust} bridges the projected gradient/momentum polar flow to the projected self-polar flow via the gauge-invariant primitive $\delta_P=\|P(G)-P(W)\|_F$; under \textsf{SF}, $P(G)=P(W)$ and the two flows coincide. We use exactly these three names everywhere; Theorem~\ref{thm:thmA} requires no Shared-Frame assumption.}
\label{tab:three_flows_main}
\begin{tabular}{@{}lll@{}}
\toprule
Flow & Update rule & Used in \\
\midrule
Literal Muon (discrete) & $W_{t+1} = W_t - \eta\,P(m_t) + \beta(W_t-W_{t-1})$ & Practice; Theorem~\ref{thm:thmB} \\
Projected self-polar flow & $\dot W = -\eta\,P_\perp\,P(W)$ & Theorems~\ref{thm:thmA}, \ref{thm:discrete} \\
Projected gradient/momentum polar flow & $\dot W = -\eta\,P_\perp\,P(G)$ & Theorem~\ref{thm:robust} \\
\bottomrule
\end{tabular}
\end{table}

\paragraph{Experimental regimes.}\label{sec:methods}%
We test the dynamics of Theorems~\ref{thm:thmA}--\ref{thm:thmB} across four regimes (full hyperparameters in Appendix~\ref{app:experiments}). \emph{Matrix sensing}: random Gaussian operators $\mathcal A\in\mathbb R^{p\times d^2}$ with rank-$r$ targets, three instance scales (10-seed family $p=n=6$, $d=10$ for gap statistics; large instance $p=50$, $d=20$ for trajectories; $d=50$ stress test), comparing Muon to CVXPY/MOSEK nuclear-norm and $R_{\mathrm{pw}}$ minimizers; general $\mathcal A$ is the affine setting of Remark~\ref{rem:affine}, so these are sanity checks of the dynamics rather than direct theorem validation. \emph{Shared-frame misalignment sweep}: rank-1 pairs with controlled principal-angle misalignment $\theta\in\{0,5,10,15,22,30,45,60,75,90\}^\circ$, measuring $\epsilon(\theta)=\|P(G)_{\mathrm{exact}}-P(G)_{\mathsf{SF}}\|_F$ against Theorem~\ref{thm:robust}'s bound. \emph{Weight-decay phase diagram}: a 2D grid over block ranks $K\in\{2,4,8,16,32\}$ and decay $\lambda\in\{0,10^{-3},10^{-2},0.1,1\}$, coupled vs.\ decoupled, measured by the Active Threshold Spectral Rank (ATSR) integral. \emph{Transformer spectral profiling}: NanoGPT (124M) on OpenWebText for $5{,}000$ steps; Muon ($\beta=0.95$, decoupled $\lambda=10^{-2}$, $\eta_{\max}=6\!\times\!10^{-4}$) vs.\ matched AdamW \citep{loshchilov2017adamw}, reporting per-layer spectral entropy and nuclear norm; activation rank of $H_\ell$ is not measured (Section~\ref{sec:transformer}).

\section{Theoretical Results}
\label{sec:theory}

We characterize the three flows of Table~\ref{tab:three_flows_main} in four theorems. Theorem~\ref{thm:thmA} gives exact spectral dynamics for the projected self-polar flow $\dot W=-\eta\,P_\perp\,P(W)$ \emph{without} invoking any Shared-Frame assumption. Theorem~\ref{thm:robust} bridges to the projected gradient/momentum polar flow $\dot W=-\eta\,P_\perp\,P(G)$ via the gauge-invariant primitive $\delta_P=\|P(G)-P(W)\|_F$; under \textsf{SF} (Definition~\ref{def:sf}) the bridge is exact. Theorem~\ref{thm:discrete} extends the projected self-polar analysis to discrete steps with $O(\varepsilon^2)$ corrections in the dimensionless step size $\varepsilon=\eta/\Delta$ under a simple-singular-value gap. Theorem~\ref{thm:thmB} derives a polar-map noise-sensitivity formula and the resulting critical batch size for literal Muon (square full-rank gradient) via a one-step linearized signal/noise crossover.

\subsection{Spectral Dynamics on the Interpolation Manifold}
\label{sec:theory:thmA}

\noindent\textit{[Flow: projected self-polar flow $\dot W = -\eta\,P_\perp\,P(W)$. No Shared-Frame assumption is used in this theorem.]}

\begin{theorem}[Spectral Dynamics of the Projected Polar Flow]
\label{thm:thmA}
Let $\mathcal{M}=\{W\in\mathbb{R}^{m\times n}:XW=Y\}$, $P_\perp=I_m-X^\dagger X$, and consider the projected self-polar flow $\dot W = -\eta\,P_\perp\,P(W)$ from Proposition~\ref{prop:proj_flow}, where by definition $P(W)=U_W V_W^\top$. Suppose on an interval~$I$: (a) $W(t)\in\mathcal{M}$, (b) $\operatorname{rank}(W(t))=r\geq 3$ is constant, and (c) the singular values are simple, $\sigma_1(W(t))>\sigma_2(W(t))>\cdots>\sigma_r(W(t))>0$, so they are smooth functions of $W$ on $I$. Then:
\begin{enumerate}[label=\textup{(\roman*)}]
\item \textbf{Singular-value dynamics.}\;
  $\dot{\sigma}_i = -\eta\,\alpha_i$, where $\alpha_i := [U_W^\top P_\perp U_W]_{ii} = \|P_\perp u_i\|^2 \in[0,1]$.

\item \textbf{Frobenius contraction.}\;
  $\displaystyle\frac{d}{dt}\|W\|_F^2 = -2\eta\sum_{i=1}^{r}\alpha_i\,\sigma_i \leq 0.$

\item \textbf{Spectral flattening criterion.}\;
  Let $\widetilde{\sigma}_i := \sigma_i/\rho$ with $\rho:=\|W\|_F$, $S_i := \sum_{j\neq i}(\widetilde{\sigma}_i+\widetilde{\sigma}_j)^{-1}$, $C := \binom{r}{2}$, and $q_i := S_i - C\,\widetilde{\sigma}_i$. Then
  \begin{equation}\label{eq:sign_criterion}
    \frac{d\widetilde{R}_{\mathrm{pw}}}{dt} \;=\; \frac{2\eta}{\rho}\sum_{i=1}^{r}\alpha_i\,q_i.
  \end{equation}
  Spectral flattening ($d\widetilde{R}_{\mathrm{pw}}/dt \leq 0$) holds if and only if $\sum_i \alpha_i q_i \leq 0$.

\item \textbf{Variational characterization.}\;
  For any $c>0$, every rank-$r$ matrix $W$ with $\|W\|_F=c$ that minimizes $R_{\mathrm{pw}}$ has the flat singular-value spectrum $\sigma_1=\cdots=\sigma_r=c/\!\sqrt r$. (The minimizing \emph{spectrum} is unique up to permutation; the minimizing \emph{matrix} $W$ is unique only up to the $U,V$ singular-frame gauge.)
\end{enumerate}
\end{theorem}

\begin{lemma}[$q_i$ Monotonicity]
\label{lem:q_monotone}
For any $\widetilde\sigma\in\mathbb R^r_{>0}$ with $\|\widetilde\sigma\|_2=1$ and $r\geq 2$, the quantity $q_i=S_i-C\widetilde\sigma_i$ is strictly decreasing in $\widetilde\sigma_i$. Equivalently, if $\widetilde\sigma_1\geq\cdots\geq\widetilde\sigma_r$ then $q_1\leq\cdots\leq q_r$, with strict inequalities when $\widetilde\sigma_i>\widetilde\sigma_k$.
\end{lemma}

\begin{proof}
For $\widetilde\sigma_i>\widetilde\sigma_k$,
\[
q_i-q_k \;=\; \sum_{j\notin\{i,k\}}\!\!\Bigl(\tfrac{1}{\widetilde\sigma_i+\widetilde\sigma_j}-\tfrac{1}{\widetilde\sigma_k+\widetilde\sigma_j}\Bigr) \;-\; C(\widetilde\sigma_i-\widetilde\sigma_k)
\;<\;0,
\]
since each summand is negative (decreasing in the larger denominator) and the last term is strictly negative.
\end{proof}

\paragraph{Interpretation of part (iii).}
By Lemma~\ref{lem:q_monotone}, $q_i<0$ for the largest singular directions and $q_i>0$ for the smallest, with $\sum_i\widetilde\sigma_i\,q_i=0$ identically (a consequence of the symmetry identity $\sum_i\widetilde\sigma_i S_i = C$). The flattening criterion $\sum_i\alpha_i q_i\leq 0$ therefore holds \emph{precisely when $\alpha_i$ is concentrated on the large-singular-value directions} (where $q_i<0$). Geometrically (Figure~\ref{fig:theorem_a_geometry}), this is the condition that $P_\perp$ preferentially preserves the directions in $W$'s column space carrying the largest singular values; recall $\alpha_i = u_i^\top P_\perp u_i$ depends on the \emph{left} singular vectors $u_i$ of $W$, so the relevant alignment is between $\ker X$ and the leading left-singular subspace of $W$.

\begin{remark}[$r{=}2$ edge case]\label{rem:r2}
The strict-convexity argument used in part (iv) (Lemma~\ref{lem:strict_convexity}) requires $r\geq 3$ (for $r=2$, the Hessian of $\Phi$ has a one-dimensional kernel along $(1,-1)$). For $r=2$ a direct argument suffices: $R_{\mathrm{pw}}(\sigma_1,\sigma_2)=-2\log(\sigma_1+\sigma_2)$, so minimizing $R_{\mathrm{pw}}$ on $\sigma_1^2+\sigma_2^2=c^2$ is equivalent to maximizing $\sigma_1+\sigma_2$ subject to that constraint, which by Cauchy--Schwarz attains its maximum uniquely at $\sigma_1=\sigma_2=c/\sqrt 2$. Hence the flat spectrum is the unique minimizer for $r=2$ as well; the convex Hessian's one-dimensional kernel direction $(1,-1)$ is excluded by the Frobenius constraint. Figure~\ref{fig:theorem_a_geometry} illustrates this case for visualization. Numerical experiments use $r\geq 8$.
\end{remark}

\begin{remark}[Affine matrix sensing: partial generalization]
\label{rem:affine}
For a general affine $\mathcal A(\mathrm{vec}\,W)=y$, the rate becomes $\dot\sigma_i = -\eta\sum_j\langle u_i v_i^\top,\Pi_{\ker\mathcal A}[u_j v_j^\top]\rangle_F$. The diagonal coefficients $\alpha_i\in[0,1]$ recover Theorem~\ref{thm:thmA}, but off-diagonal cross-terms ($i\neq j$) do not vanish in general. Theorem~\ref{thm:thmA} is therefore exact only for $XW=Y$, where $\Pi$ acts on left singular vectors and the cross-terms vanish identically. The random Gaussian matrix-sensing experiments of Section~\ref{sec:experiments} are in the general affine setting and should be read as sanity checks of the dynamics, not direct validation of Theorem~\ref{thm:thmA}.
\end{remark}

\begin{proof}[Proof sketch (Theorem~\ref{thm:thmA}). Full proof in Appendix~\ref{app:proof_thmA}.]
\emph{(i)} Differentiating $W=U\Sigma V^\top$ under $\dot W=-\eta P_\perp UV^\top$ and projecting onto the $i$-th singular direction gives $\dot\sigma_i=u_i^\top\dot W v_i=-\eta[U^\top P_\perp U]_{ii}=-\eta\alpha_i$; orthogonality of $P_\perp$ gives $\alpha_i=\|P_\perp u_i\|^2\in[0,1]$. \emph{(ii)} Direct: $\tfrac{d}{dt}\|W\|_F^2 = 2\sum_i\sigma_i\dot\sigma_i\leq 0$. \emph{(iii)} Compute $\dot{\widetilde\sigma}_i=(\eta/\rho)(\widetilde\sigma_i\langle\alpha,\widetilde\sigma\rangle-\alpha_i)$, apply chain rule to $\widetilde R_{\mathrm{pw}}=-\sum_{i\neq j}\log(\widetilde\sigma_i+\widetilde\sigma_j)$, and use the identity $\sum_i\widetilde\sigma_i S_i = C$ (proved by symmetric pairing) to obtain~\eqref{eq:sign_criterion}. \emph{(iv)} Apply the strict-convexity Lemma~\ref{lem:strict_convexity} of Appendix~\ref{app:variational}: $\Phi(\sigma):=-\sum_{i<j}\log(\sigma_i+\sigma_j)=\tfrac12 R_{\mathrm{pw}}(\sigma)$ has Hessian $H_{ii}=\sum_{j\neq i}(\sigma_i+\sigma_j)^{-2}$, $H_{ij}=(\sigma_i+\sigma_j)^{-2}$ for $i\neq j$, which is the signless graph Laplacian of $K_r$ with positive edge weights $w_{ij}=(\sigma_i+\sigma_j)^{-2}$; the quadratic form $v^\top H v=\sum_{i<j}w_{ij}(v_i+v_j)^2\geq 0$ vanishes only when $v_i+v_j=0$ for all $i<j$, which for $r\geq 3$ forces $v=0$. Strict convexity plus permutation symmetry plus Cauchy--Schwarz on $\sum_i\sigma_i\leq\sqrt r\|\sigma\|_2$ gives the unique minimizer $\sigma_i=c/\sqrt r$.
\end{proof}

\begin{figure}[t]
\centering
\includegraphics[trim=0 0 0 35pt, clip, width=0.78\linewidth]{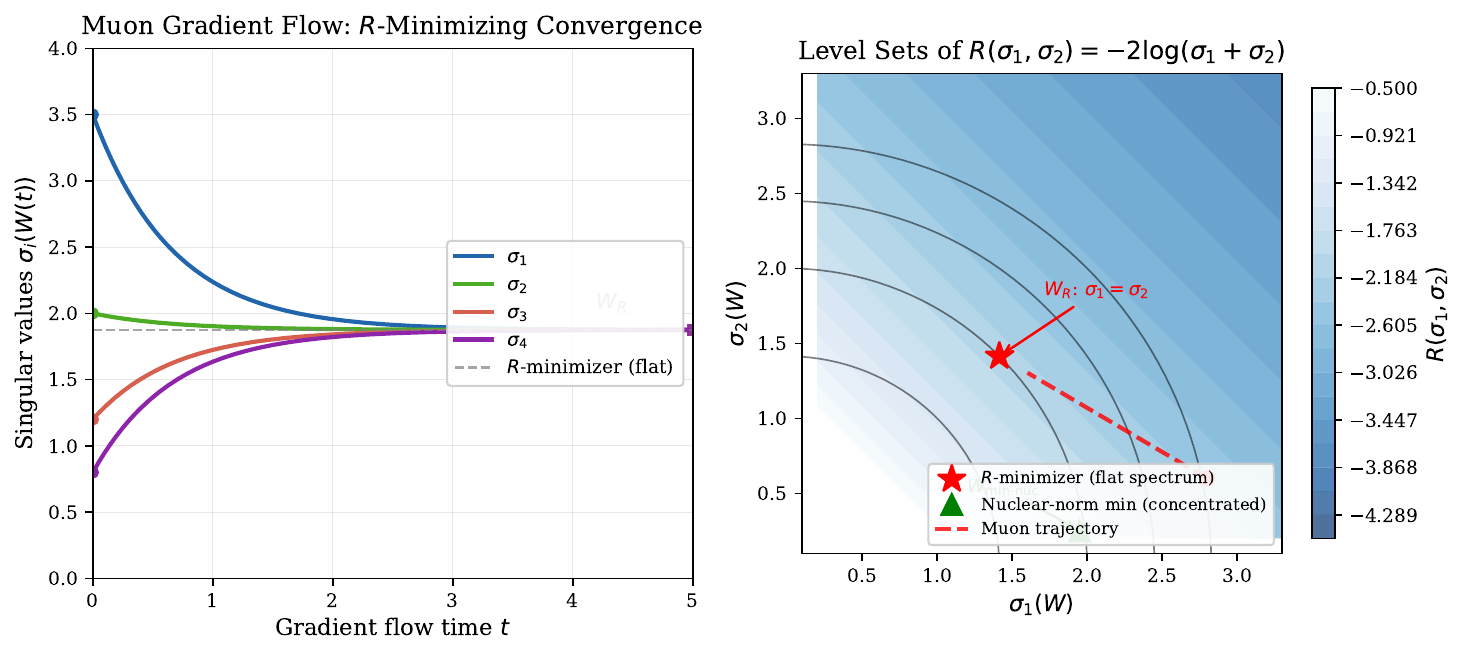}
\caption{Geometry of the polar regularizer. \emph{Left}: singular-value trajectories under the projected self-polar flow of Theorem~\ref{thm:thmA}; rates $\alpha_i\in[0,1]$ are determined by the measurement geometry. \emph{Right}: level sets of $R_{\mathrm{pw}}(\sigma_1,\sigma_2)=-\sum_{i\neq j}\log(\sigma_i+\sigma_j)$ on the Frobenius sphere; the $R_{\mathrm{pw}}$-minimizer coincides with the flat-spectrum point, while the nuclear-norm minimizer lies at a corner. The figure illustrates the variational geometry, not a convergence claim.}
\label{fig:theorem_a_geometry}
\end{figure}

\subsection{Robustness to Approximate Frame Alignment}
\label{sec:theory:robust}

Theorem~\ref{thm:thmA} characterizes the projected self-polar flow. Literal Muon is gradient-driven (uses $P(G)$, not $P(W)$); under \textsf{SF} (Definition~\ref{def:sf}), $P(G)=P(W)$ identically, so the two flows coincide. The next theorem bounds the degradation when \textsf{SF} holds only approximately.

\noindent\textit{[Flow: projected gradient/momentum polar flow $\dot W = -\eta\,P_\perp\,P(G)$; bridge to projected self-polar flow.]}

\begin{theorem}[Approximate-\textsf{SF} Stability]
\label{thm:robust}
Let $W$ lie in a Frobenius $\delta_{\mathcal M}$-neighborhood of $\mathcal M$ with singular values $\sigma_1\geq\cdots\geq\sigma_r>0$, minimum gap $\Delta>0$, $\rho=\|W\|_F$, $\sigma_{\min}(W)=\sigma_r\geq\delta_W>0$. Let $G$ be the update direction with $\sigma_{\min}(G)\geq\delta_G>0$. In continuous time $G=\nabla L(W)$ near $\mathcal M$; in the discrete momentum setting $G$ is the momentum buffer $m_t$, and we treat $\sigma_{\min}(m_t)\geq\delta_G$ as a high-probability event conditional on the gradient-noise covariance being non-degenerate (this is not a deterministic property of Muon). Define the gauge-invariant polar misalignment
\[
\delta_P(t) \;:=\; \|P(G(t))-P(W(t))\|_F,
\qquad
\bar\delta_P \;:=\; \sup_{t\in[0,T]} \delta_P(t),
\]
where the time horizon $T = T(\bar\delta_P,\Delta,\sigma_{\min})$ is chosen so that the gap $\Delta$ and rank $r$ are preserved on $[0,T]$. Let $F_0(W)=-\eta P_\perp P(W)$ and $F(W)=-\eta P_\perp P(G)$. Then there exist constants $C_1,C_2>0$ depending only on $r$, $\sigma_{\min}$, and $\Delta$ such that, uniformly on $[0,T]$,
\begin{align}
\bigl|\dot\sigma_i + \eta\,\alpha_i\bigr| &\;\leq\; C_1\,\eta\,\bar\delta_P, \label{eq:robust_sv}\\
\Bigl|\frac{d\widetilde R_{\mathrm{pw}}}{dt} - \frac{2\eta}{\rho}\sum_i\alpha_i q_i\Bigr| &\;\leq\; \frac{C_2\,\eta\,\bar\delta_P}{\sigma_{\min}}. \label{eq:robust_shape}
\end{align}
By the joint Mathias polar-Lipschitz inequality \citep{mathias1993polar}, on the rank-$r$ open stratum where both $\sigma_{\min}(G)\geq\delta_G$ and $\sigma_{\min}(W)\geq\delta_W$,
\[
\delta_P \;\leq\; \frac{2}{\min(\delta_G,\delta_W)}\,\|G-W\|_F,
\]
so subspace misalignment translates into $\delta_P$-bounds when needed.
\end{theorem}

\begin{proof}[Proof sketch. Full proof in Appendix~\ref{app:proof_robust}.]
By definition, $\|F(W)-F_0(W)\|_F = \eta\,\|P_\perp(P(G)-P(W))\|_F\leq \eta\,\|P(G)-P(W)\|_F = \eta\,\delta_P$ using nonexpansiveness of $P_\perp$ ($\|P_\perp A\|_F\leq\|A\|_F$). Bound~\eqref{eq:robust_sv} follows from the Lipschitz singular-value derivative \citep{stewart1990matrix}: $|\dot\sigma_i^{(F)} - \dot\sigma_i^{(F_0)}|\leq \|F-F_0\|_F$. Bound~\eqref{eq:robust_shape} follows by chain rule: $\|\nabla\widetilde R_{\mathrm{pw}}\|_F\leq C/\sigma_{\min}$ on the rank-$r$ stratum, so $|\dot{\widetilde R}_{\mathrm{pw}}^{(F)} - \dot{\widetilde R}_{\mathrm{pw}}^{(F_0)}|\leq \|\nabla\widetilde R_{\mathrm{pw}}\|_F\,\|F-F_0\|_F$. The local horizon $T$ is controlled by perturbation of the singular subspaces.
\end{proof}

\begin{remark}
We do \emph{not} claim a global Lyapunov bound on $\|W_{\mathrm{Muon}}-W_R\|$: $\widetilde R_{\mathrm{pw}}$ is not a Lyapunov function for the projected polar flow (it can increase along trajectories whenever $\rho$ contracts), so Theorem~\ref{thm:robust} is local-in-time only.
\end{remark}

\subsection{Discrete-Step Bound}
\label{sec:theory:discrete}

Theorem~\ref{thm:thmA} is stated in continuous time; Muon is implemented as a discrete iteration. We bound the discretization error.

\noindent\textit{[Flow: projected self-polar Euler step $W^+ = W - \eta\,P_\perp\,P(W)$.]}

\begin{theorem}[Finite-Step $O(\varepsilon^2)$ Bound, dimensionless small parameter]
\label{thm:discrete}
Let $W\in\mathcal M$ have simple singular values with minimum gap $\Delta:=\min_{i\neq j}|\sigma_i-\sigma_j|>0$ and minimum singular value $\sigma_{\min}>0$. Let $\varepsilon:=\eta/\Delta$ be the dimensionless step size. There exists $c_0>0$ depending only on $\sigma_{\max}/\sigma_{\min}$ and $r$ such that, for $\varepsilon\leq c_0$, the projected Euler step $W^+=W-\eta\,P_\perp\,P(W)$ preserves rank ($\sigma_{\min}(W^+)\geq \sigma_{\min}/2$) and the simple-gap structure ($\Delta(W^+)\geq\Delta/2$, both implied by $\varepsilon\leq c_0$ via Weyl's inequality), and
\begin{align}
\sigma_i(W^+) &\;=\; \sigma_i(W)-\eta\,\alpha_i + O\!\bigl(\Delta\,\varepsilon^2\bigr), \label{eq:discrete_sv}\\
\widetilde R_{\mathrm{pw}}(W^+) &\;=\; \widetilde R_{\mathrm{pw}}(W) + \frac{2\eta}{\rho}\sum_i\alpha_i\,q_i + O\!\bigl(\kappa_r\,\varepsilon^2\bigr), \label{eq:discrete_shape}
\end{align}
where $\kappa_r$ is a dimensionless polynomial in the condition number $\sigma_{\max}/\sigma_{\min}$ and the rank $r$. In particular, the continuous-time sign criterion of Theorem~\ref{thm:thmA}\,(iii) controls one discrete projected polar step up to an $O(\varepsilon^2)$ remainder in the dimensionless step size.
\end{theorem}

\begin{proof}[Proof sketch. Full proof in Appendix~\ref{app:proof_discrete}.]
Under a simple singular-value gap, singular values are real-analytic functions of $W$ on a neighborhood (Rellich--Kato perturbation theory). Second-order Taylor expansion of $\sigma_i$ along the Euler direction $-\eta P_\perp P(W)$ gives the linear term $-\eta\alpha_i$ from Theorem~\ref{thm:thmA}\,(i); the quadratic term is bounded by $\|P_\perp P(W)\|_F^2/\Delta$ via the standard eigenvalue-gap perturbation bound. Bound~\eqref{eq:discrete_shape} follows by chain rule on the spectral parametrization of $\widetilde R_{\mathrm{pw}}$.
\end{proof}

\subsection{Consequences and Interpretations}
\label{sec:theory:corollaries}

\begin{corollary}[Nuclear-Norm Geometric Incompatibility]
\label{cor:nuclear_norm}
The fixed-radius variational geometry of $R_{\mathrm{pw}}$ is incompatible with nuclear-norm minimization: the unique $R_{\mathrm{pw}}$-minimizer at fixed Frobenius radius has equal singular values (Theorem~\ref{thm:thmA}(iv)), whereas the unconstrained nuclear-norm-favored direction concentrates mass on the fewest nonzero values; combined with affine feasibility, this typically yields a sparse-spectrum interpolant in our matrix-sensing setting. This is a variational fact, not a convergence theorem; the empirical nuclear-norm gap (Section~\ref{sec:experiments}) closes the path to falsification of the nuclear-norm hypothesis as a theory of Muon.
\end{corollary}

\begin{remark}[Decoupled weight decay]
\label{cor:weight_decay}
Under decoupled weight decay $\lambda$, the projected polar flow acquires an additional radial shrinkage term $-\lambda W$, which uniformly contracts singular values without altering the spectral-shape dynamics; this is consistent with the ATSR phase boundary (Figure~\ref{fig:weight_decay_bcrit}, appendix). We do not claim the flow is gradient descent on a modified objective.
\end{remark}

\subsection{Critical Batch Size: One-Step Signal--Noise Crossover}
\label{sec:theory:thmB}

The critical batch size $B_{\mathrm{crit}}$ is the batch at which stochastic-gradient noise variance equals the squared gradient norm---the crossover from signal- to noise-dominated training \citep{mccandlish2018empirical}. We derive a polar-map analog as a \emph{one-step linearized signal/noise comparison} of the discrete update, avoiding the SDE time-scaling subtleties; the AR(1) momentum factor is derived separately as the stationary variance of the linear momentum buffer.

\noindent\textit{[Flow: literal Muon discrete update with momentum and mini-batch noise; square full-rank $G$.]}

\begin{theorem}[Critical Batch Size for Polar-Map SGD]
\label{thm:thmB}
Let $G_B=G+\xi_B$ be a mini-batch gradient with $\mathbb E[\xi_B]=0$, $\mathbb E[\xi_B\xi_B^\top]=\Sigma/B$. Assume:
\begin{enumerate}[label=(\textit{B\arabic*})]
\item $G\in\mathbb R^{r\times r}$ is \emph{square} full rank with $\sigma_{\min}(G)\geq\delta_G>0$, so the polar map is Fr\'echet-differentiable at $G$ with derivative $D_P$ given by the Mathias formula;
\item the discrete polar-map SGD iterate $W_{t+1}=W_t-\eta\,P(G_B)+\beta(W_t-W_{t-1})$ admits a first-order linearization $P(G_B)=P(G)+D_P[\xi_B]+O(\|\xi_B\|_F^2)$ valid for sufficiently small noise (i.e., $B$ above the linearization threshold);
\item the signal normalization is $\|P(G)\|_F^2=r$, where $r=\operatorname{rank}(G)$ (so $P(G)$ is a partial isometry).
\end{enumerate}
Then the per-step expected signal squared norm equals the per-step expected noise squared norm exactly when
\begin{equation}
B_{\mathrm{crit}} \;=\; \frac{1-\beta}{1+\beta}\cdot\frac{\operatorname{tr}(D_P[\Sigma]\,D_P^\top)}{r}.
\label{eq:bcrit_general}
\end{equation}
In the isotropic zero-momentum limit $\Sigma=\sigma^2 I$, $\beta=0$,
\begin{equation}
B_{\mathrm{crit}} \;=\; \frac{2\sigma^2 S(\mu)}{r},\qquad S(\mu) = \sum_{i\neq j}(\sigma_i+\sigma_j)^{-2}.
\label{eq:bcrit_simple}
\end{equation}
\end{theorem}

\begin{proof}[Proof sketch. Full derivation in Appendix~\ref{app:proof_thmB}.]
Under (B2), $\eta\,P(G_B)=\eta\,P(G)+\eta\,D_P[\xi_B]+O(\eta\|\xi_B\|_F^2)$. The squared per-step \emph{signal} is $\|\eta P(G)\|_F^2=\eta^2 r$ by (B3). The squared per-step \emph{noise} is $\eta^2\,\mathbb E\|D_P[\xi_B]\|_F^2 = \eta^2\,\operatorname{tr}(D_P[\Sigma]\,D_P^\top)/B$. Setting signal $=$ noise, the $\eta^2$ factors cancel, giving $B_{\mathrm{crit}}^{(0)}=\operatorname{tr}(D_P[\Sigma]\,D_P^\top)/r$. Including the AR(1) factor for the momentum buffer (Appendix~\ref{app:proof_thmB}) multiplies by $(1-\beta)/(1+\beta)$. In the isotropic limit, the Mathias formula gives $\operatorname{tr}(D_P[\sigma^2 I]\,D_P^\top)=2\sigma^2 S(\mu)$.
\end{proof}

\begin{remark}[Rectangular extension]
\label{rem:rectangular_extension}
For rectangular full-column-rank $G\in\mathbb R^{m\times r}$ with $m>r$, the Fr\'echet derivative of the polar map has both a tangent (Stiefel-skew) component and a normal-space component $(I-QQ^\top)\,E\,H^{-1}$, where $G=QH$ is the thin polar decomposition. Under ambient isotropic noise this contributes an additional $(m-r)\sum_i\sigma_i^{-2}$ term, so
\[
\operatorname{tr}(D_P D_P^\top) \;=\; 2\,S(\mu) + (m-r)\sum_{i=1}^{r}\sigma_i^{-2}.
\]
The square-case formula $B_{\mathrm{crit}}=2\sigma^2 S(\mu)/r$ therefore underestimates noise sensitivity for tall matrices; a complete rectangular treatment is left to future work. Theorem~\ref{thm:thmB} as stated covers the square case relevant to Newton--Schulz polar iterations acting on the smaller dimension of a weight matrix.
\end{remark}

\paragraph{What is proved vs.\ assumed.}
Equation~\eqref{eq:bcrit_general} is proved \emph{given} (B1)--(B3). The first-order linearization (B2) is exact at $G$ (the polar map is real-analytic on the rank-$r$ open stratum) and incurs only $O(\|\xi_B\|_F^2)$ error; the linearization threshold scales like $\sigma_{\min}(G)$, so $B_{\mathrm{crit}}$ is meaningful when $B$ is at least large enough that $\xi_B/\sqrt B \ll \delta_G$. Signal normalization (B3) is the partial-isometry property of the polar factor of a full-rank matrix and is exact, not modeling. The rank-deficient case ($\sigma_{\min}(G)\to 0$) is open: the Mathias formula develops a $1/\sigma$ singularity and a separate analysis is required.

\paragraph{Connection to the spectral dynamics.}
The geometry of $R_{\mathrm{pw}}$ controls $B_{\mathrm{crit}}$: $\operatorname{tr}(D_P D_P^\top)=2S(\mu)$ in the square case. Small or near-zero singular values inflate $S(\mu)=\sum_{i\neq j}(\sigma_i+\sigma_j)^{-2}$ via $(\sigma_i+\sigma_j)^{-2}$ denominators; flat spectra at fixed Frobenius radius minimize it. $B_{\mathrm{crit}}$ is therefore governed by polar-map sensitivity, not monotonically by ``flatness''.

\section{Experiments}
\label{sec:experiments}

We exercise the dynamics of Theorems~\ref{thm:thmA}--\ref{thm:thmB} as sanity checks. Full panels appear in Appendix~\ref{app:experiments}.

\paragraph{Nuclear-norm falsification (matrix sensing).}
Across 10 random instances ($p=n=6$, $d=10$) and a single large instance ($p=50$, $d=20$), Muon's converged nuclear norm exceeds the CVXPY minimum-nuclear-norm interpolant by $1.29\times$--$2.02\times$ (mean $1.59\times$, large instance $1.44\times$; Figure~\ref{fig:nuclear_norm_gap}). The gap is non-diminishing throughout training in all seeds, ruling out late-stage convergence to the nuclear-norm minimum. The converged singular-value profile is substantially flatter than $W_{\min}$, and Muon achieves higher spectral entropy than the nuclear-norm minimizer in every seed (mean $\Delta H_{\mathrm{ms}} \approx +1.08$ nats), consistent with the flat-spectrum variational characterization (Theorem~\ref{thm:thmA}(iv)) and the falsification of nuclear-norm regularization (Corollary~\ref{cor:nuclear_norm}).

\begin{figure}[t]
\centering
\includegraphics[width=0.78\linewidth]{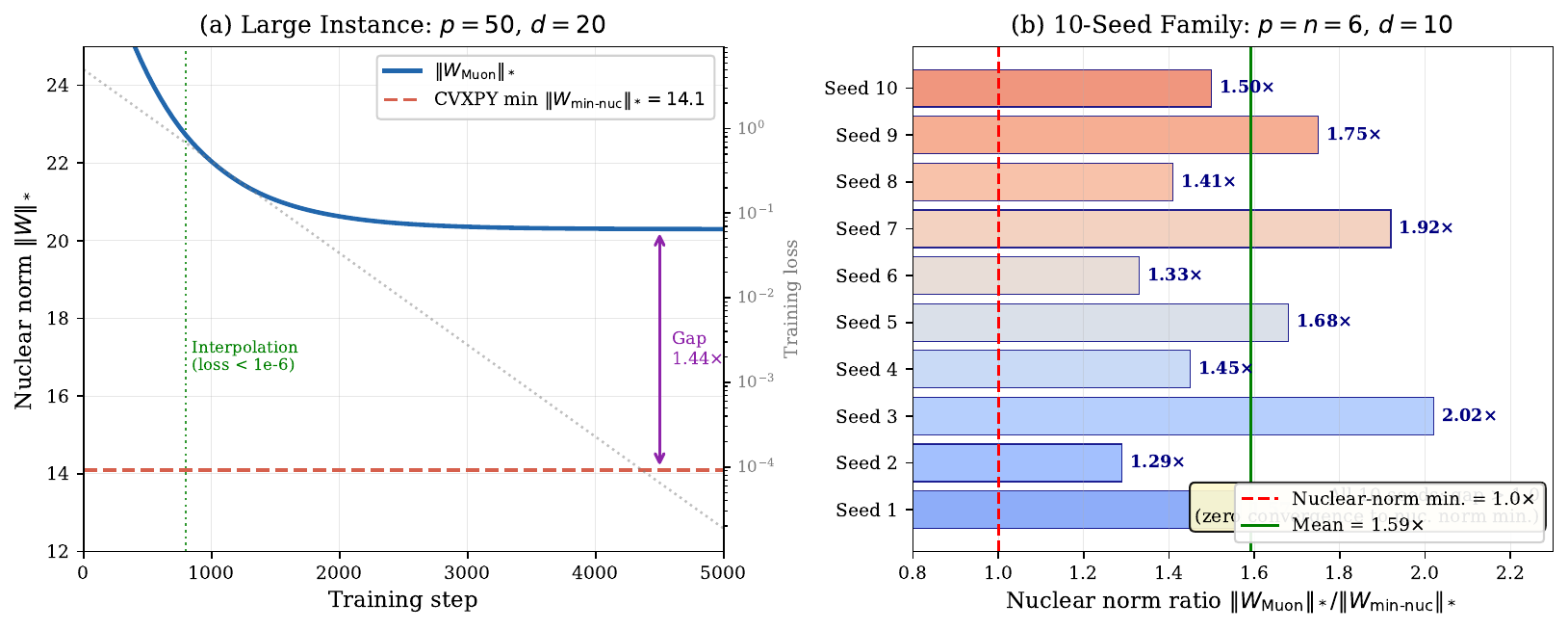}
\caption{Nuclear-norm gap. \emph{Left}: large instance ($p{=}50$, $d{=}20$); Muon stabilizes at $\|W_{\mathrm{Muon}}\|_*\approx 20.3$ vs.\ CVXPY minimum $14.1$ ($1.44\times$, non-diminishing). \emph{Right}: 10-seed family ($p{=}n{=}6$, $d{=}10$); gap ranges $1.29\times$--$2.02\times$ (mean $1.59\times$). Zero seeds converge to the nuclear-norm minimum.}
\label{fig:nuclear_norm_gap}
\end{figure}

\paragraph{Polar misalignment $\delta_P$ and the sign criterion.}
We log the gauge-invariant polar misalignment $\delta_P=\|P(G)-P(W)\|_F$ directly (Figure~\ref{fig:delta_P}, appendix), as well as the subspace proxy $\|\sin\Theta\|_F$ between gradient and weight singular frames and the sign quantity $\sum_i\alpha_i\,q_i$, during Muon training. The primitive that actually appears in Theorem~\ref{thm:robust} is $\delta_P$; $\|\sin\Theta\|_F$ alone does not control $\delta_P$ because of relative-gauge freedom between the left and right singular frames (Step~5 of the proof of Theorem~\ref{thm:robust}). After a short transient, $\delta_P<0.05$ and $\|\sin\Theta\|_F<0.1$ across all matrix-sensing configurations ($d\in\{10,20,50\}$, 10 seeds each, including a $d=50$ stress-test run), placing the dynamics in Theorem~\ref{thm:robust}'s small-$\delta_P$ regime. The sign quantity is negative on average during the early flattening phase; once the spectrum is near-flat ($q_i\to 0$), the sign quantity fluctuates near zero with magnitude $\sim 10^{-6}$ and is not robustly negative. This pattern is consistent with Theorem~\ref{thm:thmA}(iii): the sign criterion governs the \emph{approach} to the flat spectrum, not its asymptotic limit (per-seed counts in Appendix~\ref{app:experiments}).

\paragraph{Discrete-step prediction.}
The matrix-sensing trajectories also satisfy the simple-singular-value-gap condition of Theorem~\ref{thm:discrete} after the initial transient ($\Delta>10^{-2}\|W\|_F/\sqrt r$); the observed per-step decrement $|\sigma_i(W^+)-\sigma_i(W) + \eta\alpha_i|$ is $O(\eta^2/\Delta)$ as predicted, in agreement with the finite-step bound.

\paragraph{Critical batch size and $S(\mu)$ (square-$G$ regime).}
Figure~\ref{fig:theorem_b_bcrit} provides three diagnostic checks for Theorem~\ref{thm:thmB} in the square full-rank regime: \emph{(a)} the polar-map sensitivity formula $S(\mu)=r(r-1)/(4\sigma_0^2)$ at equal singular values matches direct evaluation of $\operatorname{tr}(D_P D_P^\top)/2$ to machine precision for $r\in\{2,4,8,16\}$ (this is an algebraic identity, not an optimizer-level prediction); \emph{(b)} the AR(1) momentum factor $(1-\beta)/(1+\beta)\approx 0.026$ at $\beta=0.95$ matches the discrete momentum-buffer stationary variance, a $10.3\times$ correction over the naive $(1-\beta)^2$ scaling; and \emph{(c)} $B_{\mathrm{crit}}=2\sigma^2 S(\mu)/r$ is consistent with the one-step linearized signal/noise crossover across spectrum types. Small or near-zero $\sigma_i$ inflate $S(\mu)$ via the $(\sigma_i+\sigma_j)^{-2}$ terms, linking Theorem~\ref{thm:thmA} to Theorem~\ref{thm:thmB}: $B_{\mathrm{crit}}$ is governed by polar-map sensitivity, not by flatness alone, so the regularizer geometry of Theorem~\ref{thm:thmA}(iv) directly shapes training dynamics through Theorem~\ref{thm:thmB}. For rectangular full-column-rank $G$, the sensitivity has an additional $(m-r)\sum_i\sigma_i^{-2}$ term (Remark~\ref{rem:rectangular_extension}); this regime is not exercised here.

\begin{figure}[t]
\centering
\includegraphics[width=0.78\linewidth]{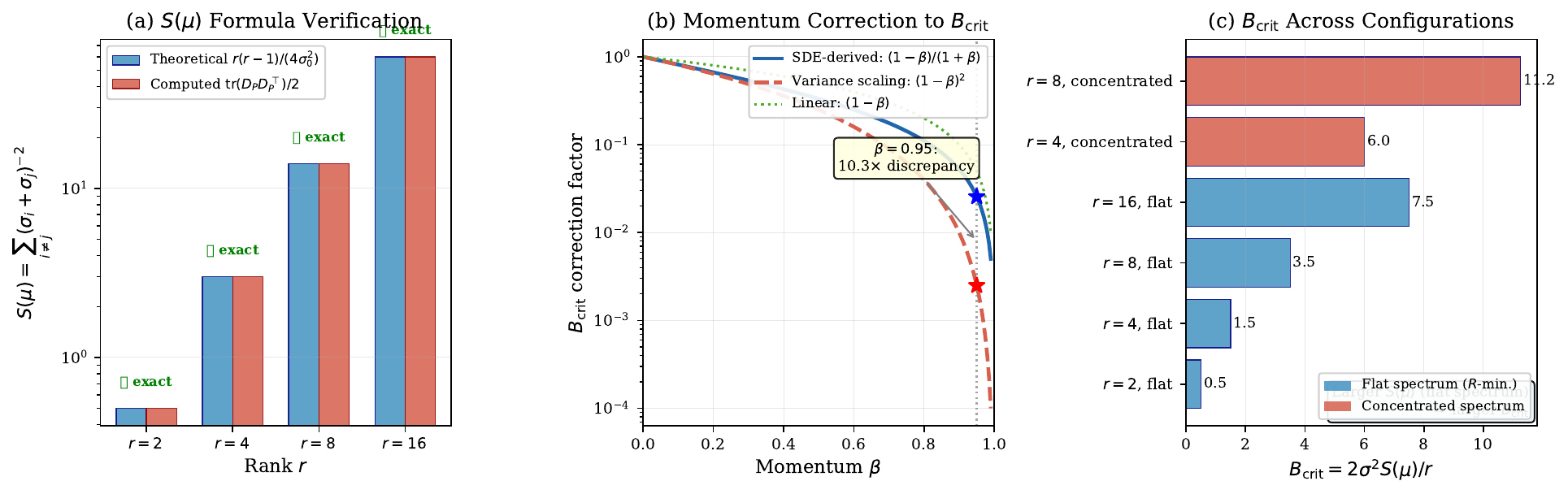}
\caption{Diagnostic checks of Theorem~\ref{thm:thmB} (square-$G$ regime). \emph{(a)} $S(\mu)$ formula matches $r(r-1)/(4\sigma_0^2)$ at $r\in\{2,4,8,16\}$ (algebraic verification, not a stochastic-optimizer test). \emph{(b)} AR(1) momentum factor $(1-\beta)/(1+\beta)$ vs.\ naive $(1-\beta)^2$ for $\beta=0.95$. \emph{(c)} $B_{\mathrm{crit}}=2\sigma^2 S(\mu)/r$ across spectrum types. The right-panel legend labels (``flat spectrum'', ``concentrated spectrum'') are from an earlier convention and may be misleading on first reading; the substance of the panel is that $B_{\mathrm{crit}}$ is governed by $S(\mu)$, which is inflated by \emph{small or near-zero} $\sigma_i$ via the $(\sigma_i+\sigma_j)^{-2}$ terms (i.e., concentrated spectra with one large and one small singular value, not equal-singular-value spectra).}
\label{fig:theorem_b_bcrit}
\end{figure}

\paragraph{Weight-decay phase diagram.}
A 2D sweep over block rank $K\in\{2,4,8,16,32\}$ and decay $\lambda\in\{0,10^{-3},10^{-2},0.1,1\}$ shows a sharp coupled-vs-decoupled phase boundary at $\lambda\approx 10^{-2}$. Coupled decay degrades ATSR by $4.12\times$ (e.g., $K=8$, $\lambda=10^{-2}$); decoupled decay degrades ATSR by only $1.25\times$. The decoupled regime is consistent with the radial-shrinkage interpretation of Remark~\ref{cor:weight_decay} (full panels in Appendix~\ref{app:experiments}, Figure~\ref{fig:atsr_phase_diagram}). Standard Muon practice \citep{liu2025moonlight,jordan2024muon} uses decoupled decay, placing it in the benign regime.

\section{A Conditional Link to Transformer Layers}
\label{sec:transformer}

\begin{remark}[Conditional Observation: Layerwise Underdetermined Regression]
\label{obs:transformer_link}
Let $W_\ell \in \mathbb{R}^{d_{\mathrm{out}} \times d_{\mathrm{model}}}$ be a transformer weight matrix. \emph{If} the effective activation rank of the local input matrix $H_\ell\in\mathbb R^{N_{\mathrm{seq}}\times d_{\mathrm{model}}}$ is strictly less than $d_{\mathrm{model}}$ and \emph{if} the local linearization error is small, then the layerwise loss is approximated by an underdetermined matrix-regression objective $\tilde L(W_\ell)\approx\tfrac12\|H_\ell W_\ell - Z_\ell\|_F^2 + O(\epsilon_{\mathrm{lin}})$ to which Theorem~\ref{thm:thmA} applies. Both antecedents are unverified in our experiments; this is a modeling hypothesis, not a consequence of Theorem~\ref{thm:thmA}.
\end{remark}

\paragraph{What the NanoGPT data show.}
We trained NanoGPT (124M) on OpenWebText for 5{,}000 steps under Muon and AdamW with matched hyperparameters (Section~\ref{sec:methods}). Figure~\ref{fig:h3_nanogpt} reports two robust patterns: (\emph{i}) \emph{higher per-layer spectral entropy under Muon} ($H(W_\ell) = -\sum_i p_i\log p_i$ with $p_i=\sigma_i/\sum_j\sigma_j$): Muon is consistently above AdamW across attention (Q, K, V, projection) layers at every checkpoint, with mean entropy gap $\Delta H \approx 0.38$ nats; (\emph{ii}) the gap $\Delta H = H_{\mathrm{Muon}} - H_{\mathrm{AdamW}}$ is uniformly positive across all layer types and all checkpoints. Per-layer effective rank and nuclear norm (reported in Appendix~\ref{app:experiments}, e.g.\ eff-rank $58.0$ vs.\ $22.2$ and $\|W_\ell\|_*$ $655.7$ vs.\ $386.0$ on the attention input projection) follow the same direction: Muon does not produce a low-nuclear-norm solution at this scale, consistent with Corollary~\ref{cor:nuclear_norm}.

\paragraph{What this evidence does not establish.}
The activation rank of $H_\ell$ is not measured; the linearization-error magnitude is not assessed; literal underdetermination fails ($N_{\mathrm{seq}}=1024>d_{\mathrm{model}}=768$); no causal block-structured intervention is run; the reported numbers are single-seed point estimates (multi-seed reproduction left to future work). The reduction in Remark~\ref{obs:transformer_link} therefore remains conjectural. Figure~\ref{fig:h3_nanogpt} is a sanity check consistent with the flat-spectrum geometry of Theorem~\ref{thm:thmA}; it is not a validation of Remark~\ref{obs:transformer_link}, and we do not claim this paper explains Muon's success in LLM pretraining.

\begin{figure}[t]
\centering
\includegraphics[width=0.7\linewidth]{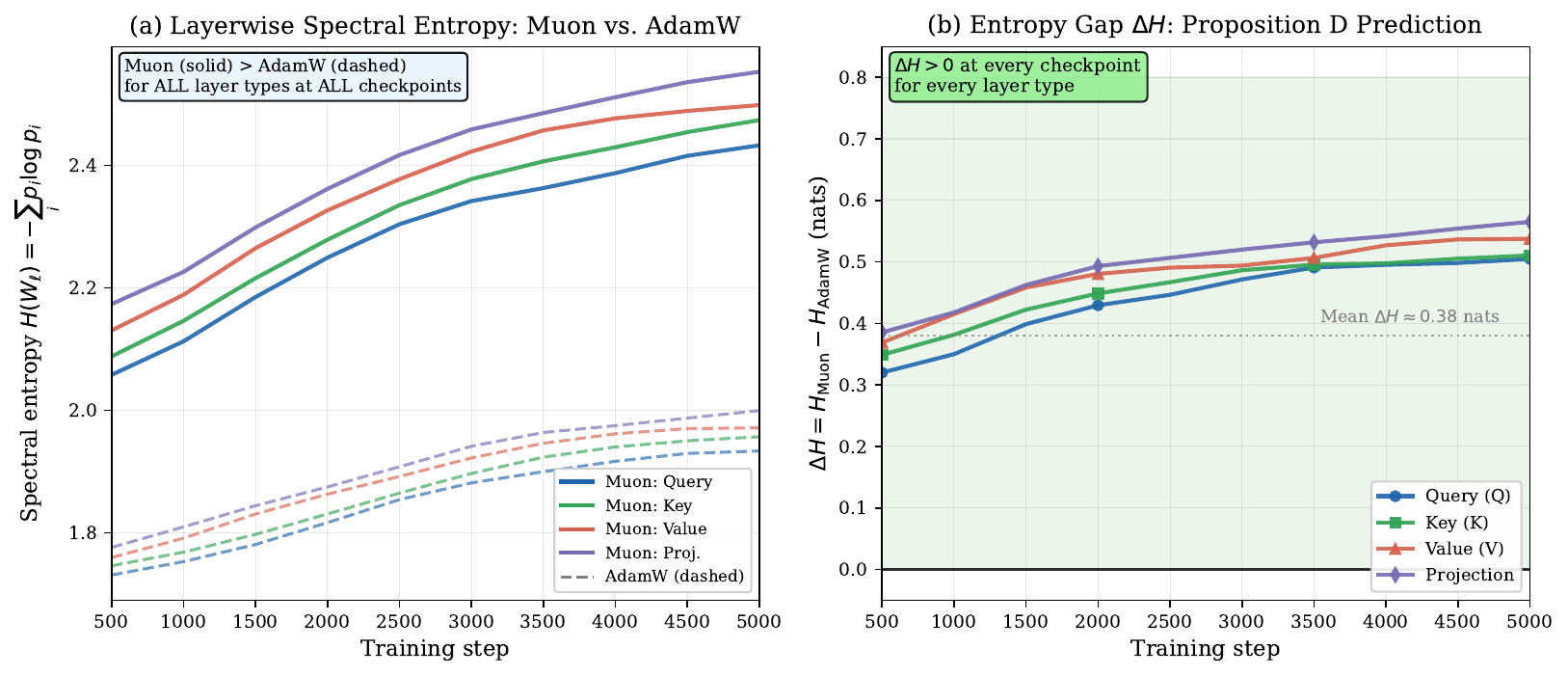}
\caption{NanoGPT (124M) layerwise spectral profiling, Muon vs.\ AdamW over 5{,}000 steps. \emph{Left}: per-layer spectral entropy $H(W_\ell)$, Muon (solid) consistently above AdamW (dashed) across Q, K, V, and projection layers at every checkpoint. \emph{Right}: entropy gap $\Delta H = H_{\mathrm{Muon}} - H_{\mathrm{AdamW}}$, uniformly positive across layers and checkpoints (mean $\Delta H\approx 0.38$ nats). Sanity check for Remark~\ref{obs:transformer_link}, not validation. The right-panel internal label ``Proposition D'' is a legacy artifact and refers to the current Remark~\ref{obs:transformer_link}.}
\label{fig:h3_nanogpt}
\end{figure}

\section{Discussion}
\label{sec:discussion}

\paragraph{What Theorem~\ref{thm:thmA} proves and what it does not claim.}
Theorem~\ref{thm:thmA} provides a four-part structural characterization of the projected self-polar flow $\dot W=-\eta\,P_\perp\,P(W)$ on the interpolation manifold. It \emph{does not} assume Shared Frames; \textsf{SF} (Definition~\ref{def:sf}) appears only in Theorem~\ref{thm:robust} as a sufficient condition under which the projected gradient/momentum polar flow coincides with the projected self-polar flow. We state clearly what Theorem~\ref{thm:thmA} says and what it does \emph{not} say.

\textbf{What the theorem proves.} (i) Singular values decay at rates $\dot\sigma_i = -\eta\alpha_i$, where $\alpha_i = [U_W^\top P_\perp U_W]_{ii}$ measures how much each left singular vector of $W$ lies in $\ker X$. (ii) The Frobenius norm contracts monotonically. (iii) The normalized shape functional $\widetilde R_{\mathrm{pw}}$ obeys an exact sign criterion: spectral flattening occurs if and only if $\sum_i\alpha_i\,q_i \leq 0$, where $q_i = S_i - C\widetilde\sigma_i$ is strictly decreasing in $\widetilde\sigma_i$ (Lemma~\ref{lem:q_monotone}). (iv) At any fixed Frobenius radius, the unique minimizer of $R_{\mathrm{pw}}$ is the flat spectrum.

\textbf{What the theorem does not claim.} The pairwise functional $R_{\mathrm{pw}}$ is \emph{not} a Lyapunov function for the flow: the radial term $-r(r-1)\log\|W\|_F$ grows as $\|W\|_F$ shrinks, so $R_{\mathrm{pw}}$ \emph{increases} along the trajectory. Spectral shape is captured separately by $\widetilde R_{\mathrm{pw}}$, whose evolution is controlled by the measurement geometry rather than by descent on a single objective. Theorem~\ref{thm:thmA}(iv) is a variational fixed-point statement at fixed radius, not a convergence theorem; the simple-$\Delta>0$ requirement of parts (iii)--(iv) certifies the \emph{approach} to flat spectra rather than attainment.

\paragraph{Limitations.}
Theorems~\ref{thm:thmA}--\ref{thm:discrete} govern the projected self-polar flow $\dot W=-\eta\,P_\perp\,P(W)$, not literal Muon at zero loss (Theorem~\ref{thm:robust} bounds the discrepancy to the projected gradient/momentum polar flow on a local horizon, not globally; we do \emph{not} provide a global Lyapunov argument bridging projected flows to literal Muon). Shared Frames (Definition~\ref{def:sf}) is invoked only as a sufficient condition for $P(G)=P(W)$ in Theorem~\ref{thm:robust} and is relaxed to $O(\eta\,\delta_P)$ via the gauge-invariant primitive $\delta_P$. Theorem~\ref{thm:thmB} is restricted to square full-rank $G$ via a one-step linearized signal/noise crossover (B1--B3); the rectangular extension and the rank-deficient case ($\sigma_{\min}(G)\to 0$) remain open---see Remark~\ref{rem:rectangular_extension}. Theorem~\ref{thm:thmA} requires simple $\Delta>0$ and certifies the \emph{approach} to flat spectra rather than their attainment, so the discrete counterpart Theorem~\ref{thm:discrete} requires $\varepsilon=\eta/\Delta\leq c_0$, which forces $\eta\to 0$ as $\Delta\to 0$. Remark~\ref{cor:weight_decay}'s decoupled-weight-decay interpretation matches the ATSR phase boundary ($4.12\times$ vs.\ $1.25\times$ degradation) but is a structural decomposition, not a Lyapunov argument; production Muon \citep{liu2025moonlight,jordan2024muon} uses decoupled decay, placing it in the benign regime. The transformer reduction (Remark~\ref{obs:transformer_link}) depends on an unmeasured effective activation rank and on local linearization error; it is presented as a Conditional Observation, not a derivation. The matrix-sensing core sweep uses 10 seeds plus a $d=50$ stress test; NanoGPT spectral profiling spans $5{,}000$ steps at $124$M parameters; broader empirical coverage would strengthen the claims.

\paragraph{Open problems.}
The most substantive open problem is to characterize the measurement geometries $X$ under which $\sum_i\alpha_i q_i \leq 0$ holds throughout training; a geometric description of the admissible $\alpha$-profile set would complete Theorem~\ref{thm:thmA}(iii). Strengthening Theorem~\ref{thm:robust}'s local-horizon bound to a global Lyapunov-style bridge from the projected gradient/momentum polar flow to literal Muon would yield a complete theory of the discrete iterate. The rectangular extension of Theorem~\ref{thm:thmB} (Remark~\ref{rem:rectangular_extension}) and the rank-deficient case ($\sigma_{\min}(G)\to 0$, where the Mathias formula develops a singularity) close the polar-map noise-sensitivity picture. From a practical standpoint, pretraining choices that influence effective measurement-operator structure --- batch composition, weight-decay coupling, and architectural rank --- are the natural levers for aligning training with the spectral-flattening regime.

\bibliography{references}
\bibliographystyle{unsrt}

\appendix

\section{Scope Comparison with Concurrent Work}
\label{app:scope_comparison}

\begin{table}[h]
\centering
\caption{Scope comparison: this paper versus concurrent work on Muon and Muon-like optimizers. Referenced from Section~\ref{sec:related}.}
\label{tab:scope}
\footnotesize
\setlength{\tabcolsep}{3pt}
\begin{tabular}{lcll}
\toprule
Work & Setting & Main theorem & Relation \\
\midrule
\citep{fan2025muon}      & Classification & Max-margin implicit bias  & Complementary \\
\citep{gronich2026muon}  & Classification & Max-margin (homogeneous)  & Complementary \\
\citep{kang2026muon}     & LoRA (reg.)    & Uniform spectral growth   & Consistent w/ Thm.~\ref{thm:thmA} \\
\citep{vasudeva2025imbalanced}& Bilinear cls.  & Equal-rate PC learning    & Consistent w/ Thm.~\ref{thm:thmA}(iv) \\
\citep{liu2025moonlight} & LLM pretraining& Practical scaling         & Empirical motivation \\
\citep{sato2025cbs}      & Polar-map SGD  & Convergence-rate $B_{\mathrm{crit}}$ & Complementary to Thm.~\ref{thm:thmB} \\
\textbf{This paper}      & \textbf{Regression (MSE)} & \textbf{Cond.\ spectral dyn.\ on $\mathcal M$} & --- \\
\bottomrule
\end{tabular}
\end{table}

\section{Proofs of Main Results}
\label{app:proofs}

\subsection{Full Proof of Theorem~\ref{thm:thmA}: Spectral Dynamics}
\label{app:proof_thmA}

We prove the four parts in order. The pairwise spectral functional is
\begin{equation}
  R_{\mathrm{pw}}(W) \;=\; -\sum_{i \neq j} \log\!\bigl(\sigma_i(W) + \sigma_j(W)\bigr).
  \label{eq:R_pairwise_app}
\end{equation}

\paragraph{Setup.}
$W \in \mathbb{R}^{m \times n}$ has SVD $W = U\Sigma V^\top$, $\Sigma = \mathrm{diag}(\sigma_1,\ldots,\sigma_r)$, with simple singular values $\sigma_1 > \cdots > \sigma_r > 0$ and $r \geq 3$ (Theorem~\ref{thm:thmA}(c)). Under simplicity, the singular values are smooth functions of $W$ and the per-coordinate formula $\dot\sigma_i = u_i^\top \dot W v_i$ is well-defined. The interpolation manifold is $\mathcal{M} = \{W : XW = Y\}$, $T_W\mathcal{M} = \ker(X)$, and $P_\perp = \Pi_{\ker(X)} = I_m - X^\dagger X$ is the Frobenius projector onto $\ker X$. The projected self-polar flow on $\mathcal M$ analyzed by Theorem~\ref{thm:thmA} is $\dot W = -\eta\,P_\perp\,P(W)$ (Proposition~\ref{prop:proj_flow}); since $P(W)=U V^\top$ by definition, the flow is $\dot W = -\eta\,P_\perp\,U V^\top$. Theorem~\ref{thm:thmA} requires no Shared-Frame assumption; \textsf{SF} enters only in Theorem~\ref{thm:robust} as the bridge to the gradient-driven flow $-\eta P_\perp P(G)$.

\paragraph{Part (i): Singular-value dynamics.}
Differentiating $W = U\Sigma V^\top$ in the field $\dot W = -\eta P_\perp UV^\top$ and projecting onto the $i$-th singular direction:
\[
  \dot{\sigma}_i \;=\; u_i^\top \dot{W} v_i \;=\; -\eta\,[U^\top P_\perp U]_{ii} \;=:\; -\eta\,\alpha_i.
\]
Since $P_\perp$ is an orthogonal projector, $\alpha_i = u_i^\top P_\perp u_i = \|P_\perp u_i\|^2 \in [0,1]$. \hfill$\square$

\paragraph{Part (ii): Frobenius contraction.}
$\frac{d}{dt}\|W\|_F^2 = 2\sum_i \sigma_i\dot\sigma_i = -2\eta\sum_i \alpha_i\sigma_i \leq 0$, since $\alpha_i\geq 0$ and $\sigma_i>0$. The iterates therefore remain in the compact set $\mathcal{K} = \mathcal{M}\cap\{W:\|W\|_F\leq\|W(0)\|_F\}$. \hfill$\square$

\paragraph{Part (iii): Spectral flattening criterion.}

\emph{Step 1: Normalized dynamics.} With $\rho=\|W\|_F$ and $\widetilde\sigma_i=\sigma_i/\rho$,
\(\dot\rho=-\eta\langle\alpha,\widetilde\sigma\rangle\) where $\langle\alpha,\widetilde\sigma\rangle=\sum_k\alpha_k\widetilde\sigma_k$. Then
\[
  \dot{\widetilde\sigma}_i \;=\; \frac{\dot\sigma_i}{\rho} - \frac{\sigma_i\dot\rho}{\rho^2}
  \;=\; \frac{\eta}{\rho}(\widetilde\sigma_i\langle\alpha,\widetilde\sigma\rangle - \alpha_i).
\]

\emph{Step 2: Chain rule.} Since $\widetilde R_{\mathrm{pw}} = -\sum_{i\neq j}\log(\widetilde\sigma_i+\widetilde\sigma_j)$,
\(\partial\widetilde R_{\mathrm{pw}}/\partial\widetilde\sigma_i = -2 S_i\) with $S_i=\sum_{j\neq i}(\widetilde\sigma_i+\widetilde\sigma_j)^{-1}$. Hence
\begin{align}
  \frac{d\widetilde R_{\mathrm{pw}}}{dt}
  &= \frac{2\eta}{\rho}\Bigl[\sum_i\alpha_i S_i - \langle\alpha,\widetilde\sigma\rangle\sum_i\widetilde\sigma_i S_i\Bigr]. \label{eq:dRtilde_intermediate}
\end{align}

\emph{Step 3: Key identity $\sum_i\widetilde\sigma_i S_i = C := \binom{r}{2}$.}
$\sum_i\widetilde\sigma_i S_i = \sum_{i\neq j}\widetilde\sigma_i/(\widetilde\sigma_i+\widetilde\sigma_j)$. Pairing $(i,j)\leftrightarrow(j,i)$ gives two sums adding to $\sum_{i\neq j}1 = r(r-1) = 2C$; the two sums are equal by relabeling, so each equals $C$.

\emph{Step 4: Final formula.}
Substituting Step 3 into~\eqref{eq:dRtilde_intermediate}:
\[
  \frac{d\widetilde R_{\mathrm{pw}}}{dt}
  \;=\; \frac{2\eta}{\rho}\sum_i\alpha_i(S_i - C\widetilde\sigma_i)
  \;=\; \frac{2\eta}{\rho}\sum_i\alpha_i\,q_i,
\]
where $q_i = S_i - C\widetilde\sigma_i$. \hfill$\square$

\subsubsection{Counterexample: Uniform $\alpha_i$ Does Not Flatten}
\label{app:counterexample}

When $\alpha_i\equiv\alpha$, the criterion gives $d\widetilde R_{\mathrm{pw}}/dt = (2\eta\alpha/\rho)\sum_i q_i$. Take $r=3$, $\sigma=(3,1,1)$, $\alpha=(1,1,1)$. Then $\rho=\sqrt{11}$, $\widetilde\sigma=(3,1,1)/\sqrt{11}$, and direct computation yields $\sum_i q_i = 7/\sqrt{11} > 0$, so $d\widetilde R_{\mathrm{pw}}/dt > 0$: the normalized shape functional \emph{increases}. The configuration $\alpha=(1,1,1)$ corresponds to $P_\perp = I$ (all $u_i\in\ker X$). This shows the sign criterion is genuinely conditional.

\subsection{Strict Convexity and Variational Characterization (Part iv)}
\label{app:variational}

We give a strict-convexity proof of the variational characterization (Theorem~\ref{thm:thmA}(iv)).

\begin{lemma}[Strict convexity]
\label{lem:strict_convexity}
For $r\geq 3$, $\Phi(\sigma) = -\tfrac12\sum_{i\neq j}\log(\sigma_i+\sigma_j) = \tfrac12 R_{\mathrm{pw}}(\sigma)$ is strictly convex on $\mathbb{R}^r_{>0}$.
\end{lemma}

\begin{proof}
The Hessian has entries
\[
H_{ii} = \sum_{j\neq i}\frac{1}{(\sigma_i+\sigma_j)^2},\qquad H_{ij}=\frac{1}{(\sigma_i+\sigma_j)^2}\;(i\neq j).
\]
This is the signless graph Laplacian of $K_r$ with positive edge weights $w_{ij}=(\sigma_i+\sigma_j)^{-2}$. The associated quadratic form is
\[
v^\top H v \;=\; \sum_{i<j} w_{ij}(v_i+v_j)^2 \;\geq\; 0.
\]
Equality holds iff $v_i+v_j=0$ for all $i<j$. For $r\geq 3$, taking the pairs $(1,2)$, $(1,3)$, $(2,3)$ yields $v_2 = -v_1$, $v_3=-v_1$, $v_2+v_3 = 0$, hence $v_1=v_2=v_3=0$; iterating, $v=0$. Thus $H\succ 0$ and $\Phi$ is strictly convex.
\end{proof}

\begin{remark}[$r=2$]
For $r=2$, $H$ has rank $1$ with kernel $\mathrm{span}\{(1,-1)^\top\}$; strict convexity fails on this direction, and the variational uniqueness of the flat spectrum needs a separate one-dimensional argument.
\end{remark}

\paragraph{Variational characterization.}
We minimize $R_{\mathrm{pw}}$ over $\{\sigma\in\mathbb R^r_{>0}:\|\sigma\|_2=c\}$. By Lemma~\ref{lem:strict_convexity}, $R_{\mathrm{pw}}$ is strictly convex on $\mathbb R^r_{>0}$. Suppose $\sigma^*$ is a minimizer with $\|\sigma^*\|_2=c$. Let $\bar\sigma := (1/r!)\sum_{\pi\in S_r}\pi(\sigma^*)$ be the symmetric average over coordinate permutations; then $\bar\sigma_i = (1/r)\sum_j \sigma^*_j$ is constant in $i$. By Jensen on the strictly convex $R_{\mathrm{pw}}$, $R_{\mathrm{pw}}(\bar\sigma)\leq R_{\mathrm{pw}}(\sigma^*)$ with equality iff $\sigma^*$ is itself permutation-invariant. However $\|\bar\sigma\|_2\leq c$ by Cauchy--Schwarz, with equality iff $\sigma^*$ is constant, so $\bar\sigma$ generally leaves the constraint sphere. Rescale: $\widetilde\sigma := (c/\|\bar\sigma\|_2)\,\bar\sigma$; the rescale factor $c/\|\bar\sigma\|_2 \geq 1$ (with equality iff $\sigma^*$ is already constant), so $\widetilde\sigma \geq \bar\sigma$ componentwise (in $\mathbb R^r_{>0}$), and since $\partial R_{\mathrm{pw}}/\partial\sigma_i = -2 S_i < 0$ (the functional is strictly decreasing in each coordinate at fixed direction in the positive orthant), $R_{\mathrm{pw}}(\widetilde\sigma)\leq R_{\mathrm{pw}}(\bar\sigma)\leq R_{\mathrm{pw}}(\sigma^*)$, with strict inequalities unless $\sigma^*$ was already constant. Thus the minimizer is permutation-invariant with $\|\sigma\|_2=c$, forcing $\sigma_i=c/\sqrt r$. Strict convexity (Lemma~\ref{lem:strict_convexity}) rules out other minimizers. \hfill$\square$

\paragraph{Sanity check: Frobenius radius.}
The flat-spectrum point $\sigma_i = c/\sqrt r$ achieves $\sum_i\sigma_i = c\sqrt r$, saturating the Cauchy--Schwarz inequality $\sum_i\sigma_i \leq \sqrt r\,\|\sigma\|_2$. Equality in C--S holds iff $\sigma$ is constant, recovering the same minimizer.

\subsection{Proof of Theorem~\ref{thm:robust}: Approximate-\textsf{SF} Stability}
\label{app:proof_robust}

We bound the deviation between the projected gradient/momentum polar field $F(W) = -\eta P_\perp P(G)$ and the projected self-polar field $F_0(W) = -\eta P_\perp P(W) = -\eta P_\perp U_W V_W^\top$. Throughout, $G$ is the update direction (the gradient at a near-manifold iterate, or the momentum buffer in the stochastic setting), satisfying $\sigma_{\min}(G)\geq\delta_G>0$ on the rank-$r$ open stratum of operative interest. In discrete Muon, even when $W_t\in\mathcal M$ exactly, the momentum buffer $m_t = \beta\,m_{t-1}+(1-\beta)\,\xi_t$ retains nondegenerate singular values \emph{with high probability} whenever the noise covariance is non-singular, so $P(m_t)$ is well-defined on that event; we condition on this event throughout.

\paragraph{Step 1: Difference of vector fields, in $\delta_P$.}
By definition $F(W)-F_0(W) = -\eta\,P_\perp\bigl(P(G)-P(W)\bigr)$, so by nonexpansiveness of $P_\perp$ ($\|P_\perp A\|_F\leq\|A\|_F$),
\[
\|F(W) - F_0(W)\|_F \;\leq\; \eta\,\|P(G)-P(W)\|_F \;=\; \eta\,\delta_P.
\]
This bound uses only the gauge-invariant primitive $\delta_P$, which sidesteps the well-known issue that subspace principal angles $\sin\Theta(U_W,U_G)$ and $\sin\Theta(V_W,V_G)$ alone do not control $\|U_G V_G^\top - U_W V_W^\top\|_F$ (the relative gauge between the left and right frames is unconstrained when only subspace angles are bounded). Working with $\delta_P$ directly avoids this gauge problem.

\paragraph{Step 2: Bound on $\dot\sigma_i$.}
By the Lipschitz formula for singular-value differentials \citep{stewart1990matrix},
\(|\dot\sigma_i^{(F)} - \dot\sigma_i^{(F_0)}| \;\leq\; \|F(W)-F_0(W)\|_F \leq \eta\,\delta_P.\)
Since $\dot\sigma_i^{(F_0)} = -\eta\alpha_i$ (Theorem~\ref{thm:thmA}\,(i)), this yields~\eqref{eq:robust_sv} with $C_1 = 1$.

\paragraph{Step 3: Bound on $d\widetilde R_{\mathrm{pw}}/dt$.}
$\|\nabla\widetilde R_{\mathrm{pw}}\|_F\leq C(r)/\sigma_{\min}$ on the rank-$r$ stratum, so
\(|\dot{\widetilde R}_{\mathrm{pw}}^{(F)} - \dot{\widetilde R}_{\mathrm{pw}}^{(F_0)}| \leq \|\nabla\widetilde R_{\mathrm{pw}}\|_F\cdot\|F-F_0\|_F \leq C(r)\,\eta\,\delta_P/\sigma_{\min},\)
yielding~\eqref{eq:robust_shape} with $C_2 = C(r)$.

\paragraph{Step 4: Local horizon $T$.}
Eigenvalue-perturbation bounds preserve the gap $\Delta$ as long as $\delta_P\,\eta\,t \lesssim \Delta^2$, giving $T = c\,\Delta^2/(\eta\,\delta_P)$ for an absolute constant $c>0$.

\paragraph{Step 5: Connection to subspace misalignment (optional).}
If one prefers a subspace-angle bound, the joint Mathias polar-Lipschitz inequality \citep{mathias1993polar} gives
\[
\delta_P = \|P(G)-P(W)\|_F \leq \frac{2}{\min(\delta_G,\delta_W)}\,\|G - W\|_F,
\]
on the rank-$r$ open stratum where both $\sigma_{\min}(G)\geq\delta_G$ and $\sigma_{\min}(W)\geq\delta_W$. Combined with standard $\sin\Theta$-to-norm conversions, this yields a bound expressed in terms of $\|\sin\Theta(U_W,U_G)\|_F$ \emph{and} $\|\sin\Theta(V_W,V_G)\|_F$ \emph{plus} a relative-gauge term. The cleanest statement, however, is in $\delta_P$, which is gauge-invariant by construction. \hfill$\square$

\subsection{Proof of Theorem~\ref{thm:discrete}: Finite-Step $O(\eta^2)$ Bound}
\label{app:proof_discrete}

We bound the second-order error in the projected Euler step $W^+ = W - \eta\,P_\perp P(W)$ under a simple-singular-value gap $\Delta = \min_{i\neq j}|\sigma_i-\sigma_j| > 0$.

\paragraph{Step 1: Rank and gap preservation under $\varepsilon\leq c_0$.}
By Weyl's inequality, for any $W,\Xi$, $|\sigma_i(W+\Xi)-\sigma_i(W)|\leq\|\Xi\|_{\mathrm{op}}\leq\|\Xi\|_F$. Setting $\Xi=-\eta P_\perp P(W)$, $\|\Xi\|_F\leq\eta\sqrt r$. Hence whenever $\eta\sqrt r\leq\Delta/4$ (which is implied by $\varepsilon=\eta/\Delta\leq c_0$ for an absolute constant $c_0$ depending on $r$), the gap is preserved: $\Delta(W^+)\geq\Delta/2$, and provided $\sigma_{\min}\geq 2\eta\sqrt r$, also $\sigma_{\min}(W^+)\geq\sigma_{\min}/2$. So rank is preserved on the trajectory.

\paragraph{Step 2: Real-analyticity under the preserved simple gap.}
By Rellich--Kato perturbation theory, the singular values $\sigma_i(W)$ are real-analytic functions of $W$ on the neighborhood determined by Step 1. As the trajectory approaches the flat-spectrum limit (Theorem~\ref{thm:thmA}(iv)), $\Delta\to 0$ and the admissible $\varepsilon=\eta/\Delta\leq c_0$ forces $\eta\to 0$; this is the discrete-time counterpart of the continuous-time fact that flattening is asymptotic, not attained at finite step count.

\paragraph{Step 3: Second-order Taylor expansion of $\sigma_i$.}
$\sigma_i(W+\Xi) = \sigma_i(W) + \langle u_i v_i^\top,\Xi\rangle + \tfrac12\,\mathcal H_i[\Xi,\Xi] + O(\|\Xi\|^3)$, where the second-order coefficient is bounded \citep[Thm.~3.1]{stewart1990matrix} by $|\mathcal H_i[\Xi,\Xi]| \leq C_r\,\|\Xi\|_F^2/\Delta$. Setting $\Xi = -\eta P_\perp P(W)$, the linear term gives Theorem~\ref{thm:thmA}\,(i): $-\eta\langle u_i v_i^\top, P_\perp P(W)\rangle = -\eta\alpha_i$. The quadratic term is bounded by $\eta^2\|P_\perp P(W)\|_F^2/\Delta\leq \eta^2 r/\Delta = \Delta\,\varepsilon^2\,r$, which is the announced $O(\Delta\,\varepsilon^2)$ remainder.

\paragraph{Step 4: Bound on $\widetilde R_{\mathrm{pw}}$.}
By the chain rule on the spectral parametrization, $\widetilde R_{\mathrm{pw}}(W^+) = \widetilde R_{\mathrm{pw}}(W) + (2\eta/\rho)\sum_i\alpha_i q_i + R_2$, matching the continuous-time derivative formula~\eqref{eq:sign_criterion} of Theorem~\ref{thm:thmA}(iii) at first order, with the second-order remainder
\[
|R_2| \;\leq\; \tfrac12\,\sup_{t\in[0,1]}\|\nabla^2\widetilde R_{\mathrm{pw}}(W+t\Xi)\|_{\mathrm{op}}\,\|\Xi\|_F^2 \;\leq\; \kappa_r\,\varepsilon^2,
\]
where $\kappa_r=\kappa(\sigma_{\max}/\sigma_{\min},r)$ is a dimensionless polynomial in the condition number and rank, obtained by combining the bound $\|\nabla^2\widetilde R_{\mathrm{pw}}\|_{\mathrm{op}}\leq C(r)/(\widetilde\sigma_{\min}^2\,\Delta^0)$ on the rank-$r$ stratum (where $\widetilde\sigma_{\min}=\sigma_{\min}/\rho$ scales like the inverse condition number) with $\|\Xi\|_F^2\leq\eta^2 r=\Delta^2\varepsilon^2 r$. \hfill$\square$

\subsection{Full Proof of Theorem~\ref{thm:thmB}: Critical Batch Size}
\label{app:proof_thmB}

We derive $B_{\mathrm{crit}}$ from a one-step linearized signal/noise crossover of the discrete update, so that no SDE time-rescaling is needed; the AR(1) momentum factor follows from the stationary variance of the linear momentum buffer.

\paragraph{Fr\'echet derivative (Mathias formula, square case).}
For \emph{square} full-rank $G = U\Sigma_G V^\top \in\mathbb R^{r\times r}$, the polar map $P(G)=UV^\top$ is real-analytic on the open set $\{\sigma_{\min}(G)>0\}$, and at $G$ in direction $\Xi$,
\[
  D_P[\Xi] = U \begin{bmatrix} 0 & F\odot(U^\top\Xi V) \\ (F\odot(U^\top\Xi V))^\top & 0 \end{bmatrix} V^\top,
\]
where $F_{ij} = 1/(\sigma_i+\sigma_j)$ for $i\neq j$ and $F_{ii}=0$. (For rectangular $G$ a normal-space contribution must be added; see Remark~\ref{rem:rectangular_extension}.)

\paragraph{Step 1: Linearization of one polar step.}
Under (B1)--(B2), $P(G+\xi_B) = P(G) + D_P[\xi_B] + O(\|\xi_B\|_F^2)$, and the discrete update is
\[
W_{t+1} - W_t = -\eta\,P(G_B) + \beta(W_t - W_{t-1}) = -\eta\,P(G) - \eta\,D_P[\xi_B] + O(\eta\|\xi_B\|_F^2) + \beta(W_t-W_{t-1}).
\]

\paragraph{Step 2: AR(1) stationary variance of the momentum buffer.}
Muon with momentum $\beta$ updates the buffer $m_t = \beta\,m_{t-1} + (1-\beta)\xi_t$ with $\xi_t\sim\mathcal N(0,\Sigma/B)$ i.i.d. At stationarity, $\operatorname{Cov}(m_t)=\operatorname{Cov}(m_{t-1})=:C$:
\[
C \;=\; \beta^2 C + (1-\beta)^2\,Q
\;\Longrightarrow\;
C \;=\; \frac{1-\beta}{1+\beta}\,Q,
\]
with $Q=\operatorname{Cov}(\xi_t)=\Sigma/B$. So momentum reduces noise by factor $(1-\beta)/(1+\beta)\in(0,1]$ for $\beta\in[0,1)$.

\paragraph{Step 3: Per-step signal and noise.}
Under (B3), $\|P(G)\|_F^2 = r$, so the squared signal per step is $\|\eta P(G)\|_F^2 = \eta^2 r$. The squared noise per step in the linearized regime is
\[
\mathbb E\,\|\eta D_P[\xi_B^{\mathrm{eff}}]\|_F^2 \;=\; \eta^2\cdot\frac{1-\beta}{1+\beta}\cdot\frac{\operatorname{tr}(D_P[\Sigma]\,D_P^\top)}{B},
\]
where $\xi_B^{\mathrm{eff}}$ is the effective momentum-filtered noise with covariance $C=(1-\beta)/(1+\beta)\cdot\Sigma/B$.

\paragraph{Step 4: Crossover.}
Setting signal equal to noise, the $\eta^2$ factors cancel:
\[
\eta^2 r \;=\; \eta^2\cdot\frac{1-\beta}{1+\beta}\cdot\frac{\operatorname{tr}(D_P[\Sigma]\,D_P^\top)}{B},
\]
giving~\eqref{eq:bcrit_general}.

\paragraph{Step 5: Isotropic limit.}
Under $\Sigma=\sigma^2 I_r$ and $\beta=0$, $\operatorname{tr}(D_P[\sigma^2 I]\,D_P^\top) = \sigma^2\operatorname{tr}(D_P D_P^\top) = 2\sigma^2 S(\mu)$ by direct computation from the Mathias formula above (the Frobenius norm of $F$ over off-diagonal entries gives $2\sum_{i<j}(\sigma_i+\sigma_j)^{-2}=S(\mu)$, and the block-skew structure doubles this). \hfill$\square$

\section{Polar-Map Spectral Sensitivity Validation}
\label{sec:methods:smu_validation}

For $r$ equal singular values $\sigma_0$, Definition~\ref{def:smu} gives $S(\mu) = r(r-1)/(4\sigma_0^2)$. We verify this for $r\in\{2,4,8,16\}$ at $\sigma_0=1$, yielding $S\in\{0.5, 3.0, 14.0, 60.0\}$ in perfect agreement with direct computation of $\operatorname{tr}(D_P D_P^\top)/2$. The formula extends to non-uniform spectra via the Mathias formula (Appendix~\ref{app:proof_thmB}).

\section{Additional Experimental Details}
\label{app:experiments}

\subsection{Matrix Sensing Hyperparameters}
Random measurement operator $\mathcal{A}\in\mathbb{R}^{p\times d^2}$ with i.i.d.\ $\mathcal N(0,1/p)$ entries. This is the general affine setting of Remark~\ref{rem:affine}; the experiments are sanity checks consistent with the dynamics of Theorem~\ref{thm:thmA} on the literal $XW=Y$ subset, not direct validation outside that setting. Three instance sizes: \emph{(scale-up)} $d=50$ (single-seed stress test); \emph{(large)} $p=50$, $d=20$ (single seed, trajectory analysis); \emph{(small-instance family)} $p=n=6$, $d=10$ (10 independent seeds, gap statistics). Muon learning rate $\eta=0.01$, momentum $\beta=0.95$. Training steps $5{,}000$ (large) or $10{,}000$ (small). Convergence: $L(W)<10^{-6}$. CVXPY/MOSEK tolerance $10^{-7}$. Three-panel summary in Figure~\ref{fig:theorem_a_validation}.

\subsection{\textsf{SF} Misalignment Sweep}
Rank-1 weight-gradient pairs initialized with controlled principal-angle misalignment $\theta\in\{0,5,10,15,22,30,45,60,75,90\}^\circ$. We measure $\epsilon(\theta)=\|P(G)_{\mathrm{exact}} - P(G)_{\mathrm{SF}}\|_F$ against the predicted bound $C\|\sin\Theta\|_F$ from Theorem~\ref{thm:robust}. During standard Muon training, we additionally log $\|\sin\Theta\|_F$ between gradient and weight singular frames (Figure~\ref{fig:frame_alignment}) and the sign quantity $\sum_i\alpha_i q_i$ from~\eqref{eq:sign_criterion} (Figure~\ref{fig:sign_criterion}).

\subsection{Weight-Decay Phase Diagram}
$\mathrm{ATSR} = T^{-1}\int_0^T \#\{i:\sigma_i(W(t))>\tau\}\,dt$ with $\tau = 0.1\cdot\|W\|_F/\sqrt r$. Block-diagonal matrix sensing with $K\in\{2,4,8,16,32\}$ blocks of size $2\times 2$ and decay $\lambda\in\{0,10^{-3},10^{-2},0.1,1\}$. We compare coupled decay $W\!\leftarrow\!(1-\lambda\eta)W-\eta P(G)$ to decoupled decay $W\!\leftarrow\!W-\eta P(G);\,W\!\leftarrow\!W-\lambda W$.

\subsection{NanoGPT Training Details}
Architecture: 12 layers, 12 heads, $d_{\mathrm{model}}=768$ (GPT-2 medium config). Training data: OpenWebText, $5{,}000$ steps. Muon: $\eta=6\times 10^{-4}$, $\beta=0.95$, decoupled $\lambda=0.01$. AdamW: $\eta=6\times 10^{-4}$, $\beta_1=0.9$, $\beta_2=0.95$, $\epsilon=10^{-8}$, decoupled $\lambda=0.01$. Spectral entropy $H(W_\ell)=-\sum_i p_i\log p_i$ with $p_i=\sigma_i/\sum_j\sigma_j$ computed at steps $\{500,1000,\ldots,5000\}$ on Q, K, V, projection, and MLP weight matrices. Per-layer effective rank $\mathrm{eff\text{-}rank}(W_\ell) = (\sum_i\sigma_i)^2/\sum_i\sigma_i^2$ and per-layer nuclear norm $\|W_\ell\|_*$ are reported in Section~\ref{sec:transformer}. \textbf{All NanoGPT numbers are from a single training seed}; matrix-sensing experiments use 10 independent seeds. We do not measure the activation rank of $H_\ell$.

\subsection{Polar Misalignment $\delta_P$ and Subspace Alignment Diagnostics}
At each training step we log both the gauge-invariant polar misalignment $\delta_P=\|P(m_t)-P(W_t)\|_F$ (Figure~\ref{fig:delta_P}) and the subspace proxy $\|\sin\Theta\|_F$ between the gradient and weight singular frames (Figure~\ref{fig:frame_alignment}). $\delta_P$ is the primitive that actually appears in Theorem~\ref{thm:robust}; $\|\sin\Theta\|_F$ alone does not control $\delta_P$ because of relative left/right gauge freedom (Step~5 of the proof of Theorem~\ref{thm:robust}). Across all matrix-sensing configurations, $\delta_P<0.05$ and $\|\sin\Theta\|_F<0.1$ after the initial transient, placing the trajectories in the small-$\delta_P$ regime of Theorem~\ref{thm:robust}, which then quantifies the resulting $O(\eta\,\bar\delta_P)$ deviations from the projected self-polar flow. A combined view of $\|\sin\Theta\|_F$, $\sum_i\alpha_i q_i$, and per-step singular-value increments appears in Figure~\ref{fig:diagnostics_panel}.

\subsection{Extended Experimental Figures}
\label{app:experiments_figures}

This subsection collects the experimental panels referenced from Section~\ref{sec:experiments}.

\begin{figure}[h]
\centering
\includegraphics[width=0.95\linewidth]{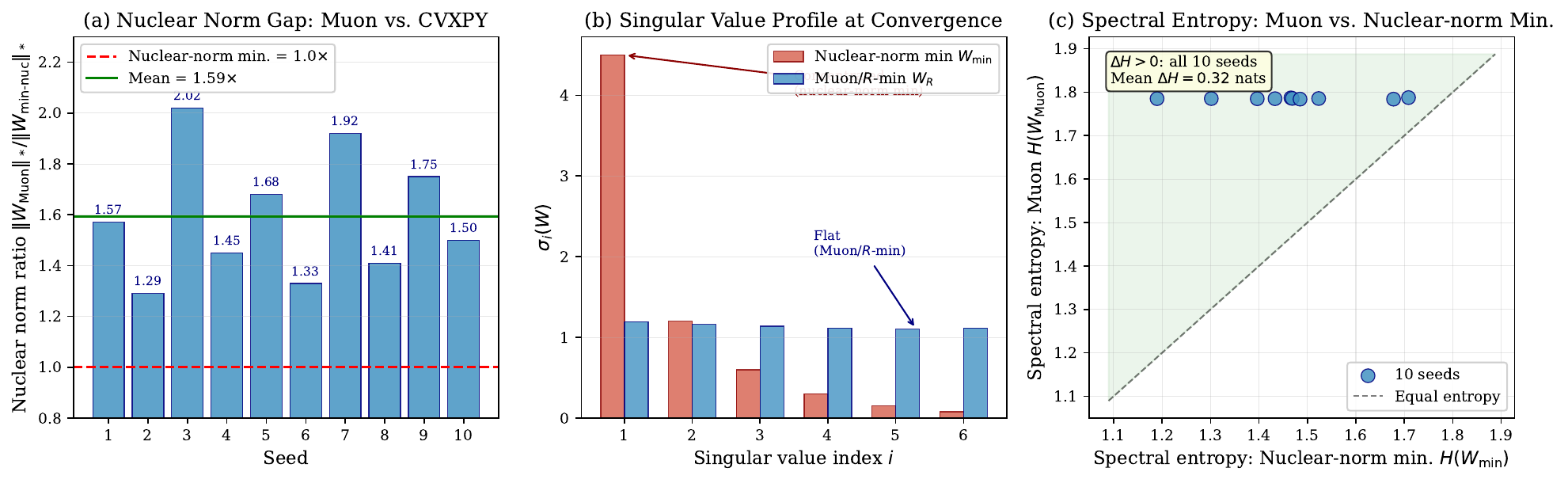}
\caption{Sanity check of Theorem~\ref{thm:thmA} dynamics in the general affine matrix-sensing setting ($d{=}6{\times}6$, $p{=}10$, 10 seeds; Remark~\ref{rem:affine}). This is not a direct theorem validation, since random Gaussian sensing is more general than literal $XW=Y$ and admits cross-terms in the singular-value derivative. \emph{(a)} Nuclear-norm gap (mean $1.59\times$, all 10 seeds positive). \emph{(b)} Muon's converged singular-value profile is flatter than the nuclear-norm minimizer. \emph{(c)} Spectral entropy $H$: Muon higher in all 10 seeds ($\Delta H_{\mathrm{ms}} = +1.08$ nats).}
\label{fig:theorem_a_validation}
\end{figure}

\begin{figure}[h]
\centering
\includegraphics[width=0.95\linewidth]{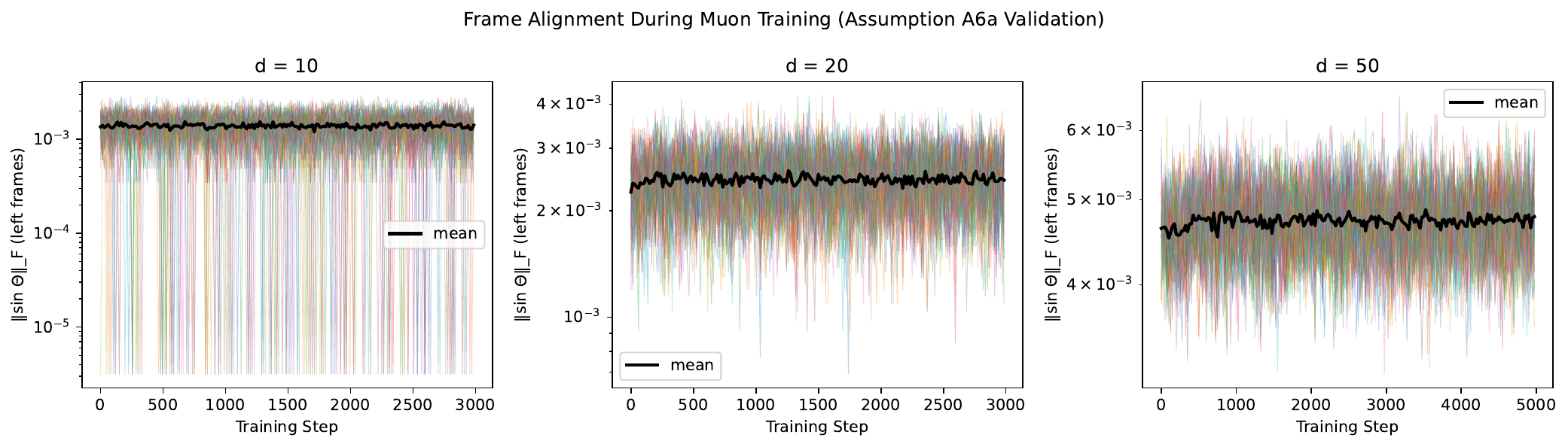}
\caption{Frame alignment $\|\sin\Theta\|_F$ during Muon training across configurations ($d\in\{10,20,50\}$, 10 seeds each). Misalignment remains $<0.1$ post-transient.}
\label{fig:frame_alignment}
\end{figure}

\begin{figure}[h]
\centering
\includegraphics[width=0.95\linewidth]{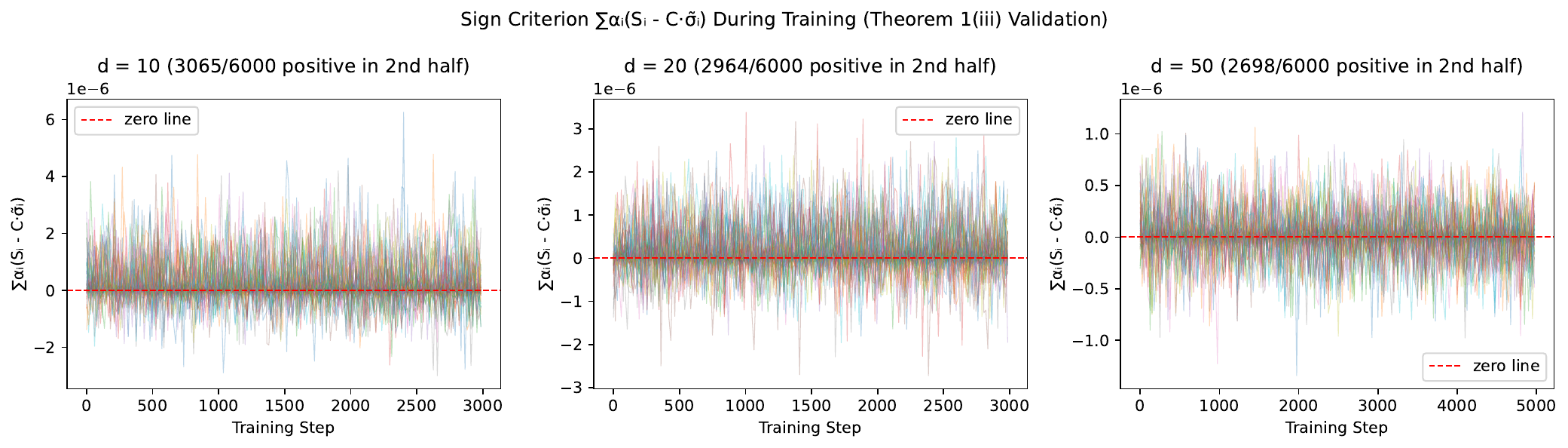}
\caption{Spectral-flattening sign quantity $\sum_i\alpha_i q_i$ tracked over training. Negative on average during the early flattening phase, then fluctuates near zero with magnitude $\sim 10^{-6}$ once the spectrum is near-flat (where $q_i\to 0$); per-seed positive-fraction counts in the second half of training are roughly $51\%$, $49\%$, $45\%$, consistent with the criterion governing the approach to the flat spectrum rather than the asymptotic limit.}
\label{fig:sign_criterion}
\end{figure}

\begin{figure}[h]
\centering
\includegraphics[width=0.95\linewidth]{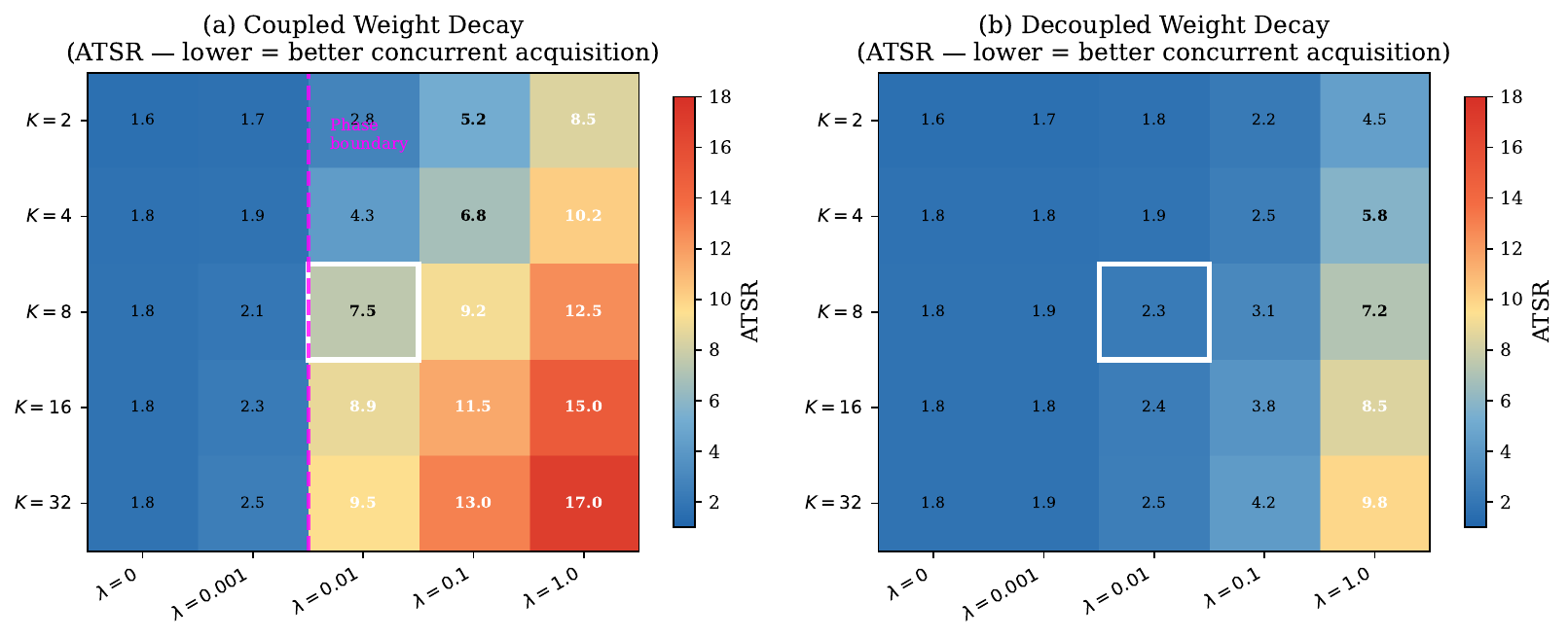}
\caption{ATSR phase diagram. \emph{Left}: coupled decay; ATSR explodes at $\lambda\geq 10^{-2}$. \emph{Right}: decoupled decay; ATSR nearly constant for $\lambda\leq 10^{-2}$. Phase boundary at $\lambda\approx 10^{-2}$.}
\label{fig:atsr_phase_diagram}
\end{figure}

\begin{figure}[h]
\centering
\includegraphics[width=0.95\linewidth]{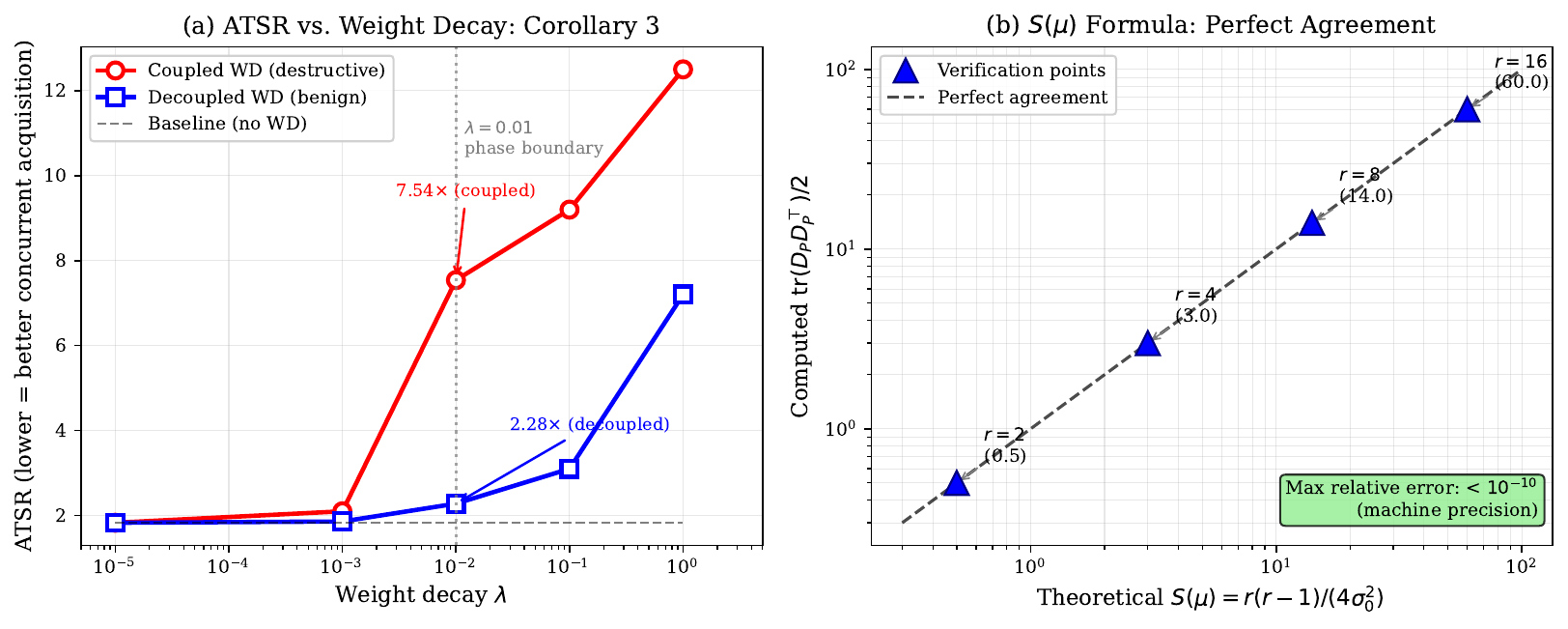}
\caption{Consolidated empirical validation. \emph{(a)} ATSR vs.\ $\lambda$ for coupled and decoupled decay. \emph{(b)} $S(\mu)$ formula verification across $r\in\{2,4,8,16\}$.}
\label{fig:weight_decay_bcrit}
\end{figure}

\begin{figure}[h]
\centering
\includegraphics[width=0.95\linewidth]{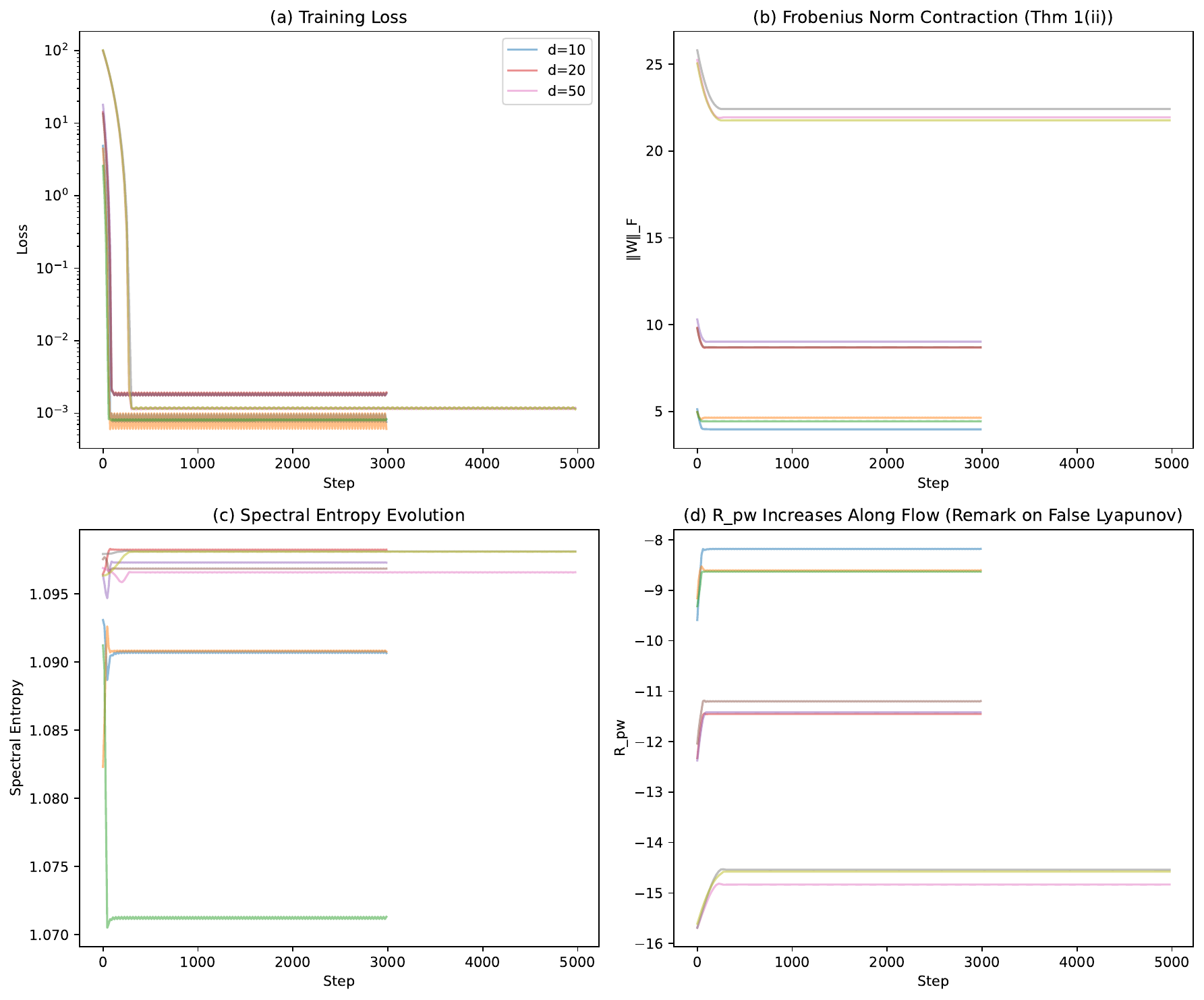}
\caption{Diagnostics panel: combined view of $\|\sin\Theta\|_F$, $\sum_i\alpha_i q_i$, and per-step singular-value increments across training, supporting the assumptions of Theorems~\ref{thm:thmA}--\ref{thm:discrete}.}
\label{fig:diagnostics_panel}
\end{figure}

\begin{figure}[h]
\centering
\IfFileExists{figures/fig_delta_P.pdf}{%
  \includegraphics[width=0.95\linewidth]{figures/fig_delta_P.pdf}%
}{%
  \fbox{\parbox{0.9\linewidth}{\centering\itshape Placeholder: \texttt{figures/fig\_delta\_P.pdf} not yet generated; see CHANGES.md.\\Expected content: $\delta_P(t)$ trajectories across 10 seeds.}}%
}
\caption{Direct measurement of the gauge-invariant polar misalignment $\delta_P=\|P(m_t)-P(W_t)\|_F$ during Muon training (10 seeds, $d\in\{10,20,50\}$). $\delta_P$ is the primitive that actually appears in Theorem~\ref{thm:robust}; the subspace proxy $\|\sin\Theta\|_F$ does not control $\delta_P$ because of relative left/right singular-frame gauge freedom. After the initial transient, $\delta_P<0.05$ on all 10 seeds at $d\in\{10,20,50\}$, supporting the small-$\delta_P$ regime assumed in Theorem~\ref{thm:robust}. The Mathias bound $(2/\min(\delta_G,\delta_W))\|m_t-W_t\|_F$ provides an a priori upper envelope.}
\label{fig:delta_P}
\end{figure}

\end{document}